\documentclass[twoside,11pt]{article}

\usepackage[preprint]{jmlr2e}

\usepackage{amsmath,bm,bbm} 
\usepackage{lipsum} 
\usepackage{xcolor} 
\usepackage{caption}
\usepackage{relsize} 
\usepackage{hyperref} 
\usepackage{changepage} 
\usepackage{algorithm}
\usepackage{algpseudocode}
\usepackage{subcaption}

\usepackage{tikz}
\usetikzlibrary{arrows.meta}

\usepackage{custom}

\ShortHeadings{Understanding Interventional TreeSHAP}
{Laberge and Pequignot}
\firstpageno{1}

\begin{document}

\title{Understanding Interventional TreeSHAP :\\
How and Why it Works.}

\author{\name Gabriel Laberge \email gabriel.laberge@polymtl.ca \\
       \addr Génie Informatique et Génie Logiciel\\
       Polytechnique Montréal\\
       \AND
       \name Yann Pequignot \email yann.pequignot@iid.ulaval.ca \\
       \addr Institut Intelligence et Données\\
       Université de Laval à Québec}


\maketitle

\begin{abstract}
Shapley values are ubiquitous in interpretable Machine Learning due to their strong theoretical background and efficient implementation in the SHAP library. Computing these values previously induced an exponential cost with respect to the number of input features of an opaque model. Now, with efficient implementations such as Interventional TreeSHAP, this exponential burden is alleviated assuming one is explaining ensembles of decision trees. Although Interventional TreeSHAP has risen in popularity, it still lacks a formal proof of how/why it works. We provide such proof with the aim of not only increasing the transparency of the algorithm but also to encourage further development of these ideas.
Notably, our proof for Interventional TreeSHAP is easily adapted to Shapley-Taylor indices and one-hot-encoded features.
\end{abstract}

\begin{keywords}
  TreeSHAP, Decision Trees, Shapley Values, Attributions, Explainability
\end{keywords}

\section{Introduction}

Ever since their introduction, Random Forests \citep{breiman2001random} and Gradient-Boosted-Trees \citep{friedman2001greedy} have
retained state-of-the-art performance for Machine Learning (ML) tasks on structured data \citep{grinsztajn2022tree}.
Still, by aggregating ensembles of decision trees
rather than considering a single one, both models do away with the inherent transparency of decision trees
in favor of more opaque decision-making mechanisms. Despite having high performance, the black box nature of these models is a considerable hurdle to their widespread applications as
humans are not inclined to use tools that they cannot understand/explain \citep{arrieta2020explainable}.

To this end, research on the subject of explaining
opaque models has been flourishing in the past few years. Notably, a method
called SHAP \citep{lundberg2017unified} has recently been invented to provide explanations of any model decisions in the form of feature attributions, where a score (positive/negative or null) is assigned to each
input feature and is meant to convey how much said feature was used in a
specific model decision. These scores are based on the well-known Shapley values
from coalitional game theory \citep{shapley1953value}. However,
computing these Shapley values for arbitrary models induces an
exponential cost in terms of the number of input features,
which inhibits their application to most ML tasks. 

The exponential burden of Shapley values has lately been
alleviated for tree ensemble models (\textit{e.g.} Random Forests and
Gradient-Boosted-Trees) via two efficient algorithms called Interventional TreeSHAP
and Conditional TreeSHAP
\citep{lundberg2020local}. The two algorithms differ in how their game theory formulation
handles inputs with missing features. The invention of these algorithms
has played a major role in making the SHAP Python library the
most popular framework for shedding light on model
decisions \citep{holzinger2022explainable}. Yet, the original paper is limited to the presentation of the algorithms, without a mathematical proof to justify their correctness. Some effort has subsequently been made to explain the main intuitions on how
Interventional TreeSHAP works, as witnessed for example in the blog of \citet{chen_blog} which contains nice interactive visualizations. Still, to the best of the author's knowledge, a detailed proof that Interventional TreeSHAP works has not yet been written. This is the purpose of this paper and we believe it is important for multiple reasons.

\begin{itemize}
     \item \textbf{Educational:} A correctness proof for an
     algorithm can be used as content for a math-focused class
     on interpretable ML.
    \item \textbf{Transparency:} It brings to light the assumptions made by the algorithm and the theoretical quantities it is computing. It alleviates the possibility
    of practitioners treating TreeSHAP as ``a black box to explain a black box''.
    \item \textbf{Transferability:} The proof for an algorithm allows one to derive proofs for similar algorithms.
\end{itemize}
The main objective of this paper is to provide the first complete
proof of how/why TreeSHAP works. We will focus on the
``Interventional'' setting of TreeSHAP, where features
are perturbed without consideration to the joint data distribution. The reason is that this algorithm is used by default any time practitioners provide a reference dataset to
TreeSHAP\footnote{\url{https://github.com/slundberg/shap/blob/45b85c1837283fdaeed7440ec6365a886af4a333/shap/explainers/\_tree.py\#L69-L72}}. Our contributions are the following:
\begin{enumerate}
    \item We provide the first detailed proof of how/why
    Interventional TreeSHAP works.
    \item The proof is shown to be effortlessly transferable to other Game Theory indices like the Shapley-Taylor indices \citep{sundararajan2020shapley}.
    Moreover, the proof is extended to features that are one-hot encoded, something that is not currently supported in the SHAP library.
    \item We provide C++ implementations for Taylor-TreeSHAP, and Partition-TreeSHAP wrapped in
    Python\footnote{\url{https://github.com/gablabc/Understand_TreeSHAP}}. 
    Crucially, this C++ implementation uses the same notation as in the paper to serve as complementary resources for the interested reader.
\end{enumerate}

The rest of the paper is structured as follows.
We begin by introducing Shapley values in \textbf{Section \ref{sec:shapley_values}}, and 
Decision Trees in \textbf{Section \ref{sec:trees}}. Then,
\textbf{Section \ref{sec:tree_shap}} presents our proof of correctness for the TreeSHAP algorithm. Next, our
theoretical results are leveraged to
effortlessly generalize TreeSHAP to the Shapley-Taylor index
in \textbf{Section \ref{sec:shap_taylor}}, and
to one-hot encoded features in \textbf{Section \ref{sec:embed}}. Finally, \textbf{Section \ref{sec:conclu}} concludes the paper.

\section{Feature Attribution via Shapley Values}\label{sec:shapley_values}

\subsection{Shapley Values}

Coalitional Game Theory studies situations where
$d$ players collaborate toward a common outcome.
Formally, letting $[d]:=\{0, 1, \ldots, d-1\}$ be the set of all $d$ players, 
this theory is concerned with games $\game:2^{[d]}\rightarrow \R$, where 
$2^{[d]}$ is the set of all subsets of $[d]$, which describes the collective 
payoff that any set of players can gain by forming a coalition. In this 
context, the challenge is to assign a credit (score) $\phi_i(\nu)\in \R$ 
to each player $i\in[d]$ based on their contribution toward the total 
value $\game([d])-\game(\emptyset)$ (the collective gain when all 
players join, taking away the gain when no one joins). Namely, such 
scores should satisfy:
\begin{equation}
    \sum_{i=1}^d\phi_i(\game) = \game([d])-\game(\emptyset).
    \label{eq:efficiency}
\end{equation}
The intuition behind this Equation is to think of the outcomes 
$\game([d])$ and $\game(\emptyset)$ as the profit of a company 
involving all employees and no employee. In that case, the
score $\phi_i(\game)$ can be seen as the salary of employee $i$ based on 
their productivity. There exist infinitely many 
score functions that respect Equation \ref{eq:efficiency}, hence it is
necessary to define other desirable properties that a score function 
should possess: Dummy, Symmetry, and Linearity.

\paragraph{Dummy} 
In cooperative games, a player $i$ is called a \emph{dummy}
if $\forall S\subseteq [d]\setminus\{i\}\,\,\,\game(S\cup\{i\})=\game(S)$. 
Dummy players basically never contribute to the game. A desirable property 
of a score is that dummy players should not be given any credit, 
\textit{i.e.} employees that do not work do not make a salary,
\begin{equation}
    \bigg[\forall S\subseteq [d]\setminus\{i\} 
    \,\,\,\game(S\cup\{i\}) = \game(S)\bigg]
    \Rightarrow \phi_i(\game)=0.
    \label{eq:dummy}
\end{equation}

\paragraph{Symmetry}
Another property is symmetry which states that players with equivalent
roles in the game should have the same score \textit{i.e.}
employees with the same productivity should have the same salary
\begin{equation}
    \bigg[\forall S\subseteq[d]\setminus\{i,j\}\quad\game(S\cup\{i\})
    =\game(S\cup\{j\})\bigg] \Rightarrow \phi_i(\game)=\phi_j(\game).
    \label{eq:symmetry}
\end{equation}

\paragraph{Linearity}
The last desirable property is that scores are linear w.r.t games, 
\textit{i.e.} if we let 
$\bm{\phi}(\nu):=[\phi_1(\nu),\phi_2(\nu),\ldots,\phi_d(\nu)]^T$, 
then for all games $\game,\mu: 2^{[d]}\rightarrow \R$ and all 
$\alpha\in \R$, we want
\begin{align}
    \bm{\phi}(\game + \mu) &= 
    \bm{\phi}(\game) + \bm{\phi}(\mu).
    \label{eq:linearity}\\
    \bm{\phi}( \alpha\mu) &= 
    \alpha\bm{\phi}(\mu).
    \label{eq:homogenity}
\end{align}
The reasoning behind this property is a bit more involved than the 
previous two. For Equation \ref{eq:linearity}, imagine that the 
two games $\nu$ and $\mu$ represent two different companies and 
that employee $i$ works at both. In that case, the salary of employee 
$i$ from company $\nu$ should not be affected by their performance 
in the company $\mu$ and vice-versa. Importantly, the total salary 
of employee $i$ ideally should be the sum of the salaries at both 
companies. For Equation \ref{eq:homogenity}, we imagine that the 
company is subject to a law-suit in the end of a quarter which 
results in a sudden reduction in profits by a factor of $\alpha$. Then,
given that each employee had fixed productivity during this quarter, 
it is only fair that all their salaries should also diminish by 
the same factor $\alpha$.

In his seminal work, Lloyd Shapley has proven the existence of a 
\textbf{unique} score function that respects Equations 
\ref{eq:efficiency}-\ref{eq:homogenity}: the so-called Shapley values.

\begin{definition}[Shapley values \citep{shapley1953value}]
Given a set $[d]:=\{1,2,\ldots,d\}$ of players
and a cooperative game $\game:2^{[d]}\rightarrow \R$, the 
Shapley values are defined as
\begin{equation}
    \phi_i(\game) = \mathlarger{\sum}_{S\subseteq [d]\setminus\{i\}} 
    W(|S|, d)\big(\game(S\cup \{i\}) - \game(S)\big)\qquad i\in[d],
\end{equation}
where 
\begin{equation} 
W(k,d):= \frac{k! (d-k-1)!}{d!}
\end{equation}
is the proportion of all $d!$ orderings of $[d]$ where a given subset 
$S$ of size $k$ appears first, directly followed by a distinguished 
element $i\notin S$.
\label{def:shapley}
\end{definition}

Intuitively, the credit $\phi_i(\game)$ is attributed to each player $i\in [d]$ based on the average contribution of adding them to coalitions $S$ that excludes them. Notice that the number of terms in
\textbf{Definition \ref{def:shapley}} is exponential in the number of players.

We record here a simple property of Shapley values with respect to dummy players that will be essential in our derivations.
\begin{lemma}[Dummy Reduction]
Let $D\subseteq [d]$ be the set of all dummies, and $D^C=[d]\setminus D$ be the set of non-dummy players, then 
\begin{equation}
    \phi_i(\game) = 
    \begin{cases}
       0 &\quad\text{if}\,\, i\in D\\
       \sum_{S\subseteq D^C \setminus\{i\}} W(|S|, |D^C|)\big(\game(S\cup \{i\}) - \game(S)\big) &\quad\text{otherwise.} 
     \end{cases}
\end{equation}
Simply put, dummy players are given no credit, and the credit of all remaining players follows the basic Shapley definition while assuming that the set of players is $D^C$ instead of $[d]$.
\label{lemma:dummy}
\end{lemma}
\begin{proof}
Let $j$ be a dummy but not $i$, we have
\begin{align*}
    \phi_i(\game) &= \mathlarger{\sum}_{S\subseteq [d]\setminus\{i,j\}} W(|S|, d)\big(\game(S\cup \{i\}) - \game(S)\big) +
    W(|S|+1, d)\big(\game(S\cup \{i,j\}) - \game(S\cup\{j\})\big)\\
    &=\mathlarger{\sum}_{S\subseteq [d]\setminus\{i,j\}} \big(W(|S|, d) + W(|S|+1, d)\big)\big(\game(S\cup \{i\}) - \game(S)\big)\\
    &=\mathlarger{\sum}_{S\subseteq [d]\setminus\{i,j\}} W(|S|, d-1)\big(\game(S\cup \{i\}) - \game(S)\big).
\end{align*}
Repeat this process for all dummies.
\end{proof}
This property is key as it implies that Shapley values of non-dummy players are unaffected by the
inclusion of an arbitrarily large number of new dummy players. It also means that, once all dummy players are identified, they can be ignored and the
Shapley value computations will only involve the set $D^C$ non-dummy players.

\subsection{Feature Attribution as a Game}
For a long time, the focus of Machine Learning (ML) has been on improving generalization performance, regardless of the cost of model interpretability. For instance, because of their state-of-the-art performance, tree ensembles are typically used instead of individual decision trees which are inherently more interpretable. However, the ML community has reached a point where the lack of transparency of the best-performing models inhibits
their widespread acceptance and deployment \citep{arrieta2020explainable}. To address this important
challenge, the research community has been working for the past few years on techniques that shed light on the decision-making of
opaque models. A popular subset of such techniques provide insight on model decisions in the form of feature attribution
\textit{i.e.} given an opaque model $h$ and an input $\bm{x}\in \R^d$, a score $\phi_i(h, \bm{x})\in \R$ is provided to each feature $i\in [d]$ in order
to convey how much feature $i$ was used in the decision $h(\bm{x})$. Given the lack of ground truth for the feature attribution for an opaque model, the community has focused on exploiting the existence and uniqueness of Shapley values to the problem of explaining ML models \citep{sundararajan2020many}.

Following the BaselineShap formulation of \cite{sundararajan2020many}, 
when explaining the model prediction $h(\bm{x})$ at some input $\bm{x}$ of interest, we can compare said prediction to the prediction $h(\bm{z})$ at a baseline value of the input $\bm{z}$. In this context, feature attributions can be formulated as feature contributions towards the gap $h(\bm{x})-h(\bm{z})$ viewed as the total value of a carefully designed coalitional game. Defining such a game requires specifying the collective payoff for every subset of features. In BaselineShap, this is achieved by defining the collective payoff for a given set $S$ of features to be the prediction of the model at the input whose value for features from $S$ are given by the point of interest $\bm{x}$, while other 
feature values are taken from the baseline point $\bm{z}$.
Formally, given an input of interest $\bm{x}$, a subset of features $S\subseteq [d]$ that are considered active, and the baseline $\bm{z}$, the replace function $\bm{r}^{\bm{z}}_S:\R^d\rightarrow \R^d$, defined as
\begin{equation}
    \replacei{S} = 
        \begin{cases}
        x_i & \text{if  } i \in S\\
        z_i & \text{otherwise},
    \end{cases}
\label{eq:replace_fun}
\end{equation}
is used to simulate the action of activating features in $S$ (or equivalently shutting down features not in $S$).
The resulting coalitional game for ML interpretability is the following.

\begin{definition}[The Baseline Interventional Game]
Comparing the model predictions at $\bm{x}$ and $\bm{z}$ can be done by computing the Shapley
values of the following game:
\begin{equation}
    \game_{h,\bm{x},\bm{z}}(S):= h(\,\replace{S}\,).
\end{equation}
This formulation is called Interventional because we intervene on feature $i$ by replacing its baseline value $z_i$ with the value $x_i$ irrespective of the value of other features. Hence we break correlations between the various features.
\label{def:explain_game}
\end{definition}
By  Equation \ref{eq:efficiency}, the resulting Shapley
values $\bm{\phi}(\game_{h,\bm{x},\bm{z}}) \equiv \shapvec{h}$ will sum to the gap
\begin{equation}
    \sum_{i=1}^d\shap{i}{h}=h(\bm{x})-h(\bm{z}).
\end{equation}
Therefore, Shapley values following this Interventional
game formulation are advertised as a mean to answer contrastive questions such as : \textit{Why does my model attribute a higher risk of
stroke to person A compared to person B?}

We finish this Section with a toy example that illustrates how Interventional Shapley values are computed. We shall be explaining a two-dimensional \texttt{AND} function
\begin{equation}
    h(\bm{x}) = \mathbbm{1}(x_0>0) \mathbbm{1}(x_1>0).
\end{equation}

Our goal is to explain the discrepancy between $h(\bm{x})=1$ with $x=(1, 1)^T$ and $h(\bm{z})=0$
with $\bm{z}=(\minus1, \minus1)^T$. Let us begin by computing
$\shap{0}{h}$ by directly plugging in the definition

\begin{equation}
\begin{aligned}
    \shap{0}{h} &= \mathlarger{\sum}_{S\subseteq \{1\}} \frac{|S|!(d-|S|-1)!}{d!}\big[\,h(\,\replace{S\cup \{0\}}\,) - h(\,\replace{S}\,)\,\big]\\
    &=\frac{1}{2}\big[\,h(\,\replace{\{0, 1\}}\,) - h(\,\replace{\{1\}}\,)\,\big] + \frac{1}{2}\big[\,h(\,\replace{\{0\}}\,) - h(\,\replace{\emptyset}\,)\,\big]\\
    &=\frac{1}{2}\big(\,h(1, 1) - h(\minus1, 1)\,\big) + \frac{1}{2}\big(\,h(1, \minus1) - h(\minus1, \minus1)\,\big)\\
    &=\frac{1}{2}\times 1 + \frac{1}{2}\times 0 =\frac{1}{2}.
\end{aligned}
\end{equation}

Same thing for $\shap{1}{h}$. We see that $\shap{0}{h}=\shap{2}{h}$, which is aligned with the symmetry property (cf. Equation \ref{eq:symmetry}), and that $\shap{0}{h}+\shap{1}{h}=1=h(\bm{x})-h(\bm{z})$. Figure \ref{fig:toy_example} shows the intuition behind this computation. We see there are two coordinate-parallel paths that go from $\bm{z}$ to $\bm{x}$. We average the contribution of changing the $i$th component across these two paths.

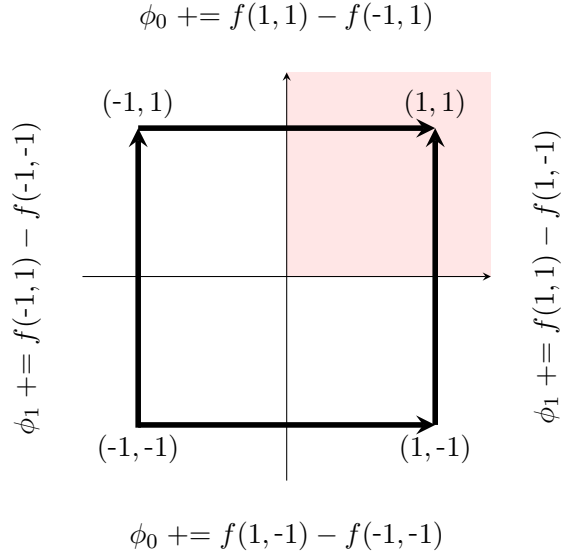
\begin{figure}[t]
    \centering
    \captionsetup{justification=centering}
    \resizebox{8.125cm}{!}{
        \resizebox{0.6\linewidth}{!}{
\begin{tikzpicture}
\draw[fill=white!90!red,draw=none]  (0, 2.75) rectangle (2.75,0);

\draw[-stealth] (-2.75, 0) -- (2.75, 0);
\draw[-stealth] (0, -2.75) -- (0, 2.75);

\node[anchor=south] at (2,2) {$(1,1)$};
\node[anchor=south] at (-2,2){$(\minus1,1)$};
\node[anchor=north] at (2,-2) {$(1,\minus1)$};
\node[anchor=north] at (-2,-2){$(\minus1,\minus1)$};

\draw[-stealth,line width=0.075cm] (-2,-2) -- node[yshift=-1.5cm] {$\phi_0 \pluseq f(1,\minus1) - f(\minus1,\minus1)$} (2,-2);
\draw[-stealth,line width=0.075cm] (-2,2) -- node[yshift=1.5cm] {$\phi_0 \pluseq f(1,1) - f(\minus1,1)$} (2,2);
\draw[-stealth,line width=0.075cm] (-2,-2.037) -- node[xshift=-1.5cm,rotate=90] {$\phi_1 \pluseq f(\minus1,1) - f(\minus1,\minus1)$} (-2,2);
\draw[-stealth,line width=0.075cm] (2,-2) -- node[xshift=1.5cm,rotate=90] {$\phi_1 \pluseq f(1,1) - f(1,\minus1)$} (2,2);

\end{tikzpicture}
}
    }
    \caption{Example of Shapley values computation.}
    \label{fig:toy_example}
\end{figure}

\section{Tree-Based Models}\label{sec:trees}

\subsection{Decision Tree}

A directed graph $G=(N, E)$ is a set of nodes $n\in N$ and edges 
$e\in E \subset N\times N$. The node at the tail of $e$ is $e_0\in N$ while the node at the head of the edge is $e_1\in N$. The node $e_1$ is called a child of $e_0$. A full binary tree is a rooted tree in which every node has exactly two children (a left one and a right one) or none. We view such a graph as a directed graph by making all its edges point away from the root. We denote a full binary tree by $T=(N,E)$ where $N=\{0,1,\ldots, |N|-1\}$ is the set of nodes and $E\subset N\times N$ is the set of directed edges and write $0 \in N$ for the root. An \emph{internal node} is a node $n$ which has two children $l<r$ and $l$ is called the left child of $n$ and $r$ is called the right child of $n$. In this case, we say that $(n,r)$ is a right edge and $(n,l)$ is a left edge. A leaf is a node with no children and all leaves are stored in the set $L\subseteq N$. 

\begin{definition}[Binary Decision Tree]
A Binary Decision Tree on $\R^d$ is full binary tree $T=(N,E)$ in which every internal node $n$ is labeled by a pair $(i_n,\gamma_n)\in [d]\times \R$ and every leaf $l\in L$ is labeled with a value $v_l$. For an internal node $n$, the label $(i_n,\gamma_n)$ encodes the boolean function $R_n:\R^d\to\{0,1\}$ given by $R_n(\bm{x}):= \mathbbm{1}(x_{i_n}\leq \gamma_n)$.
\end{definition}

See Figure \ref{fig:tree_example} for a simple example of Binary Decision Tree. Any Binary Decision Tree induces a function $h:\R^d\to \R$ defined as follows. For every $\bm{x}\in \R^d$, we start at the top of the tree (the root) and check the condition $R_0(\bm{x}):= \mathbbm{1}(x_{i_0}\leq \gamma_0)$: if it is true we
go down to the left child of the root, and if it is false we go down to the right child.
This procedure is repeated until we reach a leaf node $l$
and the model outputs $h(\bm{x})=v_l$.
For instance, in Figure \ref{fig:tree_example}, we see that for the
input $\bm{x}=(3.4, 0.2, 2)^T$, we go from the root to the
node 1 and end up at the leaf 4, so $h(\bm{x})=v_4$. We note that the input ``flowed through'' the sequence of edges $(\,(0,1),(1,4)\,)$ which goes from root to leaf. We shall refer to such a sequence as a maximal path.

\begin{definition}[Maximal Path]
A directed path $P$ in a full binary tree $T=(N,E)$ is sequence $P=(e^{[k]})_{k=0}^{\ell-1}$ of edges in $E$ such that the $e_1^{[k]} = e_0^{[k+1]}$ for all $k=0, 1,\ldots,\ell-2$. A maximal path is a directed path $P=(e^{[k]})_{k=0}^{\ell-1}$ that cannot be extended (and if it has a positive length then it starts at the root and ends at a leaf).
\label{def:path}
\end{definition}

Examples of maximal paths in this decision tree of Figure \ref{fig:tree_example} include $P=(\,(0,1), (1,4)\,)$ and $P=(\,(0,2), (2,5)\,)$.
Given the definition of maximal path, we now aim at describing the model $h$ in a more formal manner. Since our intuition about $h$ is that $\bm{x}$ flows downward on edges selected by the Boolean functions, our notation should only involve edges. Hence, for an edge
$e=(e_0,e_1)$, we define $i_e$ as
the feature index of the internal node $e_0$ at its tail.
In Figure \ref{fig:tree_example}, we have that $i_{(1,4)}=2$ for example.
Also, for each edge $e=(e_0,e_1)$ we define the Boolean function $R_e(\bm{x})$ by $R_{e_0}(\bm{x})$ if $e$ is a left edge and $1- R_{e_0}(\bm{x})$ otherwise. When $R_e(\bm{x})=1$ for some edge $e$, we shall say that $\bm{x}$ ``flows though'' the edge $e$, a terminology that will be used throughout this paper.
Now, for each maximal path $P$ ending at some leaf $l\in L$ we define $h_P:\R^d\to \R$ by 
\begin{equation}
h_P(\bm{x})=v_l\prod_{e\in P}R_e(\bm{x}).
\label{eq:h_p}
\end{equation}
This step function outputs zero unless the input flows through all the edges in the maximal path $P$.
For example, in Figure \ref{fig:tree_example}, the function $h_P$ associated with  $P=(\,(0, 1), (1,4)\,)$ is $h_P(\bm{x})=v_4\mathbbm{1}(x_1\leq 0.5)\mathbbm{1}(x_2>1.33)$.
Finally, given an input $\bm{x}$, there exists only one maximal path $P$ such that $h_P(\bm{x})\neq 0$. This is because any given $\bm{x}$ can only flow through one path $P$ from root to leaf.
Hence, we can interpret $h$ as
\begin{equation}\label{eq:DT_sumPath}
    h(\bm{x})=\sum_{P} h_P(\bm{x}),
\end{equation} 
where the sum is taken over all maximal paths in the Decision Tree.
Note that the decision tree function $h$ of Figure \ref{fig:tree_example} can be written as 
\begin{align*}
h(\bm{x})=\quad 
&v_3\mathbbm{1}(x_1\leq0.5)\mathbbm{1}(x_2\leq1.33)+
v_4\mathbbm{1}(x_1\leq 0.5)\mathbbm{1}(x_2>1.33)\\
+
&v_5\mathbbm{1}(x_1>0.5)\mathbbm{1}(x_0\leq0.25)+
v_6\mathbbm{1}(x_1> 0.5)\mathbbm{1}(x_0>0.25).
\end{align*}

\begin{figure}[t!]
    \centering
    \captionsetup{justification=centering}
    \begin{tikzpicture}

\tikzset{nodestyle/.style={scale=1.33,draw=black,shape=circle,fill=white!97!black}}
\tikzset{edgetyle/.style={-stealth, line width=1mm}}

\def\xstep{2.5};
\def\ystep{2.5};
\node[nodestyle] at (0,0) (0) {0};
\node[nodestyle] at (-\xstep,-\ystep) (1) {1};
\node[nodestyle] at (\xstep,-\ystep) (2) {2};
\node[nodestyle] at (-3*\xstep/2,-2*\ystep) (3) {3};
\node[nodestyle] at (-\xstep/2,-2*\ystep) (4) {4};
\node[nodestyle] at (\xstep/2,-2*\ystep) (5) {5};
\node[nodestyle] at (3*\xstep/2,-2*\ystep) (6) {6};

\node[yshift=1.4cm] at (0) 
{$\begin{aligned}
    i_0&=1\\ 
    \gamma_0&=0.5\\ 
    \bm{x_1}&\bm{
    \leq0.5?}
\end{aligned}$};
\node[yshift=1.4cm,xshift=-0.65cm] at (1) 
{$\begin{aligned}
    i_1&=2\\ 
    \gamma_1&=1.33\\ 
    \bm{x_2}&\bm{
    \leq1.33?}
\end{aligned}$};
\node[yshift=1.4cm,xshift=0.65cm] at (2) 
{$\begin{aligned}
    i_2&=0\\ 
    \gamma_2&=0.25\\ 
    \bm{x_0}&\bm{
    \leq0.25?}
\end{aligned}$};

\foreach \t in {3,...,6}
{
	\node[yshift=-0.75cm] at (\t) {$v_\t$};
}

\draw[edgetyle] (0) -- node[xshift=-0.25cm,yshift=0.5cm] {\textbf{Yes}} (1);
\draw[edgetyle] (0) -- node[xshift=0.25cm,yshift=0.5cm] {\textbf{No}} (2);
\draw[edgetyle] (1) -- (3);
\draw[edgetyle] (1) -- (4);
\draw[edgetyle] (2) -- (5);
\draw[edgetyle] (2) -- (6);

\draw[dashed,-stealth,line width=0.5mm, red]  
plot[smooth, tension=.7] coordinates {(0,-0.5) (-1.8,-2.5) (-1,-4.5)};

\end{tikzpicture}
    \caption{Basic example of Binary Decision Tree. In \red{red} 
    we highlight the maximal path followed by the input 
    \red{$\bm{x}=(3.4, 0.2, 2)^T$}.}
    \label{fig:tree_example}
\end{figure}
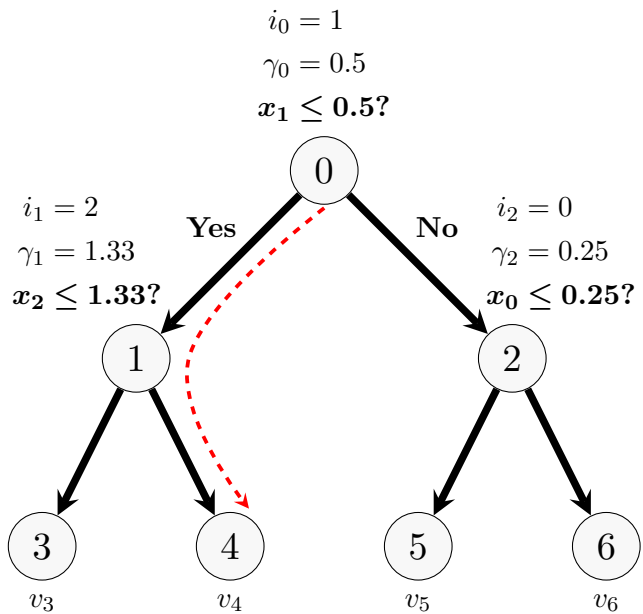
\newpage
\section{Interventional Tree SHAP}\label{sec:tree_shap}

As previously highlighted, computing Shapley values (cf. \textbf{Definition \ref{def:shapley}}) requires summing over an exponential number of subsets $S\subseteq [d]\setminus \{i\}.$ This exponential complexity previously prohibited the application of Shapley values to ML tasks which typically involve a large number of features. Nonetheless, 
by exploiting the structure of Decision Trees, Interventional TreeSHAP alleviates this exponential complexity in $d$ at the cost of no longer being model-agnostic. We now describe step-by-step how and why TreeSHAP works.

In this section, we shall assume $\bm{x}$ and $\bm{z}$ are fixed and we are concerned with computing Shapley values for a forest $\mathcal{F}$ of decision trees, namely $\bm{\phi}(\mathcal{F},\bm{x},\bm{z})$. A first application of linearity (cf. Equations \ref{eq:linearity} \& \ref{eq:homogenity}) shows that 
\begin{align}
    \bm{\phi}(\mathcal{F},\bm{x},\bm{z})
    &= \bm{\phi}\bigg(\frac{1}{|\mathcal{F}|}\sum_{h\in \mathcal{F}} h,\bm{x},\bm{z}\bigg)\\
    &=\frac{1}{|\mathcal{F}|}\sum_{h\in \mathcal{F}}\bm{\phi}(h,\bm{x},\bm{z}),
\end{align}
Henceforth, we can thus fix a decision tree $h$ and focus on how to compute $\bm{\phi}(h,\bm{x},\bm{z})$.
\subsection{Naive Tree Traversal}

Since the function represented by the decision tree $h$ can be written as a sum over all maximal paths (cf. Equation~\ref{eq:DT_sumPath}), a second application of linearity (cf. Equation \ref{eq:linearity}) allows us to derive
\begin{equation}
    \begin{aligned}
    \bm{\phi}(h,\bm{x},\bm{z})
    &= \bm{\phi}\bigg(\sum_Ph_P,\bm{x},\bm{z}\bigg)\\
    &= \sum_P\bm{\phi}(h_P,\bm{x},\bm{z}),
    \end{aligned}
    \label{eq:lin_hp}
\end{equation}
Shapley values for a decision tree are therefore equal to the sum of Shapley values over all maximal paths in the tree. This means that the feature attributions can be computed via the following tree-traversal algorithm.

\begin{algorithm}
\caption{Tree SHAP Naive Pseudo-code}
\begin{algorithmic}[1]
\Procedure{Recurse}{$n$}
    \If{$n\in L$}
         \State $\bm{\phi} \pluseq \bm{\phi}(h_P, \bm{x}, \bm{z})$\Comment{Add contribution of the maximal path}
    \Else
        \State \Call{Recurse}{$n_\text{left child}$};
        \State \Call{Recurse}{$n_\text{right child}$};
    \EndIf
\EndProcedure
\State $\bm{\phi}=\bm{0}$;
\State \Call{Recurse}{0};
\end{algorithmic}
\label{alg:simplest}
\end{algorithm}

\subsection{Compute the Shapley Value of a Maximal Path}

We can now fix a maximal path $P$ in the decision tree $h$ and focus on the Shapley values $\bm{\phi}(h_P,\bm{x},\bm{z})$.
We first identify four types of edges $\typeX,\typeZ,\typeF,\typeB$ which can occur in $P$.
\begin{definition}[Edge Type]
We say that an edge $e\in E$ is of 
\begin{equation}
    \begin{aligned}
    &\text{Type \typeX} \fsp \text{if } \fsp R_e(\bm{x})=1 \text{ and }  R_e(\bm{z})=0\\
    &\text{Type \typeZ} \fsp \text{if } \fsp R_e(\bm{x})=0 \text{ and }  R_e(\bm{z})=1\\
    &\text{Type \typeF} \fsp \text{if } \fsp R_e(\bm{x})=1 \text{ and }   R_e(\bm{z})=1\\
    &\text{Type \typeB} \fsp \text{if } \fsp R_e(\bm{x})=0 \text{ and }  R_e(\bm{z})=0.
    \end{aligned}
\end{equation}
Here \typeX stands for $\bm{x}$ flows, \typeZ stands for
$\bm{z}$ flows, \typeF stands for both Flow,  and \typeB stands for both are Blocked.

\label{def:edge_types}
\end{definition}
See Figure \ref{fig:4_types} for an informal example of each of these types of edges for a fixed maximal path $P$. Moreover, Figure \ref{fig:edges_types_example} presents the previous
example of a simple decision tree where each edge is colored with respect to its type.
From this point on in the document, we will systematically color edges w.r.t their type.

\begin{figure}[t]
    \centering
    \captionsetup{justification=centering}
    \begin{tikzpicture}
\def\step{3.3};
\foreach \x in {0,...,4}
{
	\node[scale=1.33,draw=black,shape=circle,fill=white!97!black] at (\step*\x,0) (\x) {\x};
}

\draw[color=green!70!black,-stealth, line width=1mm] (0) -- (1);
\draw[color=red,-stealth,line width=1mm] (1) -- (2);
\draw[color=blue,-stealth,line width=1mm] (2) -- (3);
\draw[-stealth,line width=1mm] (3) -- (4);

\node at (\step/2,-0.5) {\typeX};
\node at (3*\step/2,-0.5) {\typeZ};
\node at (5*\step/2,-0.5) {\typeF};
\node at (7*\step/2,-0.5) {\typeB};

\node at (1*\step/2,-1) {$\bm{x}$ flows};
\node at (3*\step/2,-1) {$\bm{x}$ does not flow};
\node at (5*\step/2,-1) {$\bm{x}$ flows};
\node at (7*\step/2,-1) {$\bm{x}$ does not flow};

\node at (1*\step/2,-1.5) {$\bm{z}$ does not flow};
\node at (3*\step/2,-1.5) {$\bm{z}$ flows};
\node at (5*\step/2,-1.5) {$\bm{z}$ flows};
\node at (7*\step/2,-1.5) {$\bm{z}$ does not flow};

\end{tikzpicture}
    \caption{All types of edges in a maximal path $P=(\,(0,1), (1,2), (2, 3), (3,4)\,)$ from root to leaf.}
    \label{fig:4_types}
\end{figure}
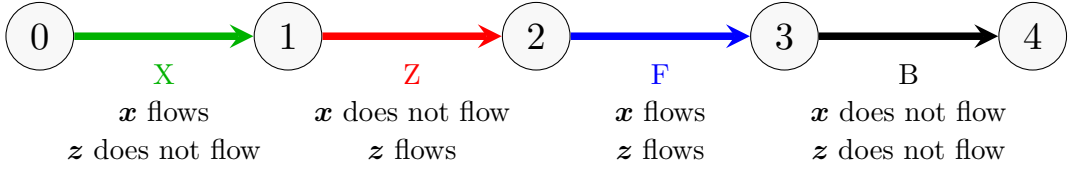

\begin{figure}[t]
    \centering
    \captionsetup{justification=centering}
    \begin{tikzpicture}

\tikzset{nodestyle/.style={scale=1.33,draw=black,shape=circle,fill=white!97!black}}
\tikzset{edgetyle/.style={-stealth, line width=1mm}}

\def\xstep{2.5};
\def\ystep{2.5};
\node[nodestyle] at (0,0) (0) {0};
\node[nodestyle] at (-\xstep,-\ystep) (1) {1};
\node[nodestyle] at (\xstep,-\ystep) (2) {2};
\node[nodestyle] at (-3*\xstep/2,-2*\ystep) (3) {3};
\node[nodestyle] at (-\xstep/2,-2*\ystep) (4) {4};
\node[nodestyle] at (\xstep/2,-2*\ystep) (5) {5};
\node[nodestyle] at (3*\xstep/2,-2*\ystep) (6) {6};

\node[yshift=1.25cm] at (0) 
{$\begin{aligned}
    i_0&=1\\ 
    \gamma_0&=0.5
\end{aligned}$};
\node[yshift=1.25cm,xshift=-0.6cm] at (1) 
{$\begin{aligned}
    i_1&=2\\ 
    \gamma_1&=1.33
\end{aligned}$};
\node[yshift=1.25cm,xshift=0.6cm] at (2) 
{$\begin{aligned}
    i_2&=0\\ 
    \gamma_2&=0.25
\end{aligned}$};

\foreach \t in {3,...,6}
{
	\node[yshift=-0.75cm] at (\t) {$v_\t$};
}

\draw[edgetyle,color=blue] (0) -- (1);
\draw[edgetyle] (0) --  (2);
\draw[edgetyle,color=green!70!black] (1) -- (3);
\draw[edgetyle,color=red] (1) -- (4);
\draw[edgetyle, color=blue] (2) -- (5);
\draw[edgetyle] (2) -- (6);

\end{tikzpicture}
    \caption{Example of types of edges in a Decision Tree. The input and reference inputs are $\bm{x}=(0, 0, 1)^T$ and $\bm{z}=(-2, -1, 2)^T$.}
    \label{fig:edges_types_example}
\end{figure}
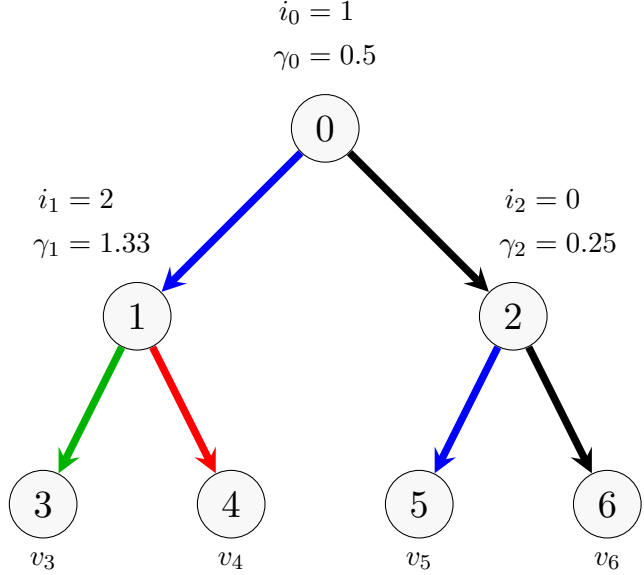
\begin{lemma}
If $P$ contains an edge of type $\typeB$, then 
$$\forall\, S\subseteq [d]\quad h_P(\,\replace{S}\,)=0.$$
\label{lemma:avoid_D}
\end{lemma}
\begin{proof}
Let $e$ be an edge of type $\typeB$ in $P$. Then for every $S\subseteq [d]$, we have $R_e(\replace{S})=0$, which implies that $h_P(\replace{S})\propto R_e(\replace{S})=0$.
\end{proof}

\begin{corollary}
If $P$ contains an edge of \typeB, 
then $\shapvec{h_P} = \bm{0}$.
\label{corollary:shap_avoid_D}
\end{corollary}
\begin{proof}
It follows from \textbf{Lemma \ref{lemma:avoid_D}} and the linearity property of Shapley values (cf. Equation \ref{eq:homogenity} with $\alpha=0$). Indeed, by definition, linear functions map zero to zero so the
Shapley values of the null game are null.
\end{proof}

\begin{figure}[t]
    \centering
    \resizebox{0.9\textwidth}{!}{
\begin{tikzpicture}
\def\step{2.95};
\foreach \x in {0,...,6}
{
	\node[scale=1.33,draw=black,shape=circle,fill=white!97!black] at (\step*\x,0) (\x) {\x};
}

\draw[color=green!70!black,-stealth, line width=1mm] (0) -- (1);
\draw[color=red,-stealth,line width=1mm] (1) -- (2);
\draw[color=blue,-stealth,line width=1mm] (2) -- (3);
\draw[color=green!70!black,-stealth,line width=1mm] (3) -- (4);
\draw[color=green!70!black,-stealth,line width=1mm] (4) -- (5);
\draw[color=red,-stealth,line width=1mm] (5) -- (6);

\node[yshift=1.35cm] at (0) 
{$\begin{aligned}
    i_0&=1\\ 
    \gamma_0&=-0.5\\ 
    \bm{x_1}&\bm{
    > \!-0.5?}
\end{aligned}$};

\node[yshift=1.35cm] at (1) 
{$\begin{aligned}
    i_1&=2\\ 
    \gamma_1&=1.5\\ 
    \bm{x_2}&\bm{
    > 1.5?}
\end{aligned}$};

\node[yshift=1.35cm] at (2) 
{$\begin{aligned}
    i_2&=1\\ 
    \gamma_2&=1\\ 
    \bm{x_1}&\bm{
    \leq 1?}
\end{aligned}$};

\node[yshift=1.35cm] at (3) 
{$\begin{aligned}
    i_3&=0\\ 
    \gamma_3&=-1\\ 
    \bm{x_0}&\bm{
    > \!-1?}
\end{aligned}$};

\node[yshift=1.35cm] at (4) 
{$\begin{aligned}
    i_4&=1\\ 
    \gamma_4&=-0.33\\ 
    \bm{x_1}&\bm{
    > \!-0.33?}
\end{aligned}$};

\node[yshift=1.35cm] at (5) 
{$\begin{aligned}
    i_5&=0\\ 
    \gamma_5&=-1.5\\ 
    \bm{x_0}&\bm{
    \leq -1.5?}
\end{aligned}$};

\end{tikzpicture}
}
    \caption{Example of sets $S_X=\{0,1\}$ and $S_Z=\{0,2\}$ for $\bm{x}=(0, 0, 1)^T$ and $\bm{z}=(-2, -1, 2)^T$.}
    \label{fig:SA_SB}
\end{figure}
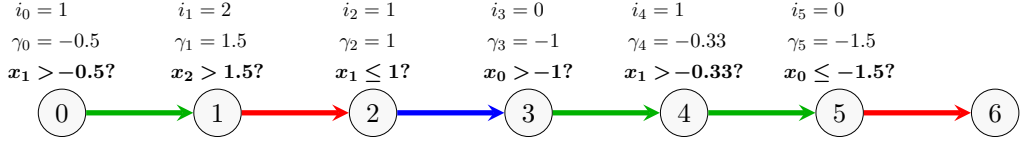

Now, assuming $P$ does not
contain any type \typeB edges, we have
\begin{equation}
\begin{aligned}
    h_P(\replace{S})&=v_l \prod_{e\in P}R_e(\replace{S}) \\
    &= v_l \prod_{
    \substack{e\in P\\
    e \text{ type \typeX}}}\!\!\!
    R_e(\replace{S}) \!\!\!
    \prod_{
    \substack{e\in P\\
    e \text{ type \typeZ}}}\!\!\!
    R_e(\replace{S})
    \underbrace{
    \prod_{\substack{e\in P\\
    e \text{ type \typeF}}}\!\!\!
    R_e(\replace{S})}_{=1\,\,\forall \,S}\\
    &= v_l \prod_{
    \substack{e\in P\\
    e \text{ type \typeX}}}\!\!\!
    \mathbbm{1}(i_e\in S)\!\!\!
    \prod_{
    \substack{e\in P\\
    e \text{ type \typeZ}}}\!\!\!
     \mathbbm{1}(i_e\notin S).
\end{aligned}
\label{eq:h_s_boolean}
\end{equation}
For an edge $e$, recall that $i_e$ denotes the feature index at its tail $e_0$.
The last line follows from the observation that, for an edge $e$ of type \typeX,
the only times $R_e(\replace{S})=1$ are when the $i_{e}$th component of
$\replace{S}$ is set to $\bm{x}$ and not $\bm{z}$ (\textit{i.e.} $i_{e}\in S$). Alternatively, for an edge $e$ of type \typeZ, $R_e(\replace{S})=1$ if and
only if the $i_{e}$th component of
$\replace{S}$ is set to $\bm{z}$ and not $\bm{x}$ (\textit{i.e.} $i_{e}\notin S$).

We define the three following sets 
\begin{equation}
    \begin{aligned}
    S_X:&=\{i_{e} : e\in P,  \,\,\text{and $e$ is of type $\typeX$}\},\\
    S_Z:&=\{i_{e} : e\in P, \,\,\text{and $e$ is of type $\typeZ$}\}.\\
    S_{XZ}&:=S_X\cup S_Z.
    \end{aligned}
    \label{eq:SX_SZ}
\end{equation}
See Figure \ref{fig:SA_SB} for an example of sets $S_X$ and $S_Z$ on some maximal path.
It is important to realize that the sets $S_X$, $S_Z$, and $S_{XZ}$ depend on the current maximal
path $P$. Still, we assume that $P$ is fixed and hence we
do not make the dependence explicit in the notation. 

\begin{lemma}
    If $ S_X \cap S_Z\neq \emptyset $, then 
    \begin{equation}
        \forall\, S\subseteq [d] \quad h_P(\,\replace{S}\,) = 0.
    \end{equation}    

    \label{lemma:disjoint}
\end{lemma}
\begin{proof}
Let $j$ be a feature common to both $S_X$ and $S_Z$. Then, there must exists
an edge $e$ of \typeX and an edge $e'$ of type \typeZ such that $j=i_e=i_{e'}$. From Equation \ref{eq:h_s_boolean}, we have
$$
h_P(\,\replace{S}\,)\propto \mathbbm{1}(j\in S)\mathbbm{1}(j\notin S)=0.$$


\end{proof}
\vspace{-0.85cm}
\begin{corollary}
    If $S_X\cap S_Z \neq \emptyset$, then $        \bm{\phi}(h_P, \bm{x},\bm{z}) = \bm{0}.$
    \label{corollary:shap_disjoint}
\end{corollary}
\begin{proof}
 By \textbf{Lemma \ref{lemma:disjoint}} and Linearity of the Shapley values (cf. Equation \ref{eq:homogenity}).
 \end{proof}

We observe that \textbf{Corollary \ref{corollary:shap_avoid_D} \& \ref{corollary:shap_disjoint}} provide sufficient
conditions for the contributions at line
3 of Algorithm \ref{alg:simplest} to be null. We now provide a result that highlights which features are dummies of the Coalitional Game for the maximal path $P$.

\begin{lemma}
    If a feature does not belong to $S_{XZ}$, then it is a dummy feature for the Baseline Interventional Game $\game_{h_P, \bm{x}, \bm{z}}$ (cf. Definition \ref{def:explain_game}).
    \label{lemma:SAB_dummy}
\end{lemma}



\begin{proof}
If we assume there exists a type \typeB edge in $P$, then by 
\textbf{Lemma \ref{lemma:avoid_D}}, the game is null and all players are dummies.
In the non-trivial case where $P$ does not contain a type \typeB edge, we must let $i\notin S_{XZ}$ be an index, and
$S\subseteq [d]\setminus\{i\}$ be an arbitrary subset of players that excludes $i$,
and show that $h_P(\,\replace{S\cup\{i\}}\,) = h_P(\,\replace{S}\,)$. Since
$i \notin S_{XZ}$, there are no edges $e$ of type \typeX or \typeZ such that
$i_{e}=i$. Moreover, since $i\neq i_{e}$ implies the equivalence
$i_{e}\in S\iff i_{e}\in S\cup\{i\}$, we have
\begin{align*}
    h_P(\replace{S})
    &= v_l \prod_{
    \substack{e\in P\\
    e \text{ type \typeX}}}\!\!\!
    \mathbbm{1}(i_e\in S)\!\!\!
    \prod_{
    \substack{e\in P
    e \text{ type \typeZ}}}\!\!\!
     \mathbbm{1}(i_e\notin S)\\
     &= v_l \prod_{
    \substack{e\in P\\
    e \text{ type \typeX}}}\!\!\!
    \mathbbm{1}(i_e\in S\cup\{i\})\!\!\!
    \prod_{
    \substack{e\in P\\
    e \text{ type \typeZ}}}\!\!\!
     \mathbbm{1}(i_e\notin S\cup\{i\})\\
     &=h_P(\replace{S\cup\{i\}}),
\end{align*}


which concludes the proof.
\end{proof}

By \textbf{Lemma \ref{lemma:dummy}}, we note that computing the Shapley values $\bm{\phi}(h_P,\bm{x},\bm{z})$ will only require evaluating the
game $\nu_{h,\bm{x},\bm{z}}(S):=h_P(\,\replace{S}\,)$ for coalitions $S\subseteq S_{XZ}$. Additionally, we will assume w.l.o.g that the current maximal path contains no type \typeB edges and that $S_X$ and $S_Z$ are disjoint. Indeed, failing to reach these requirements will simply cause the Shapley values to be zero.

\begin{lemma}[Flow Blocking]
If $P$ contains no type \typeB edges and the sets $S_X$ and $S_Z$ are disjoint, then for any $S\subseteq S_{XZ}$ we have
\begin{equation}
    h_P(\,\replace{S}\,) = 
    \begin{cases}
       v_l &\quad\text{if}\,\, S=S_X\\
       0 &\quad\text{otherwise,} 
     \end{cases}
\end{equation}
where $l\in L$ is the leaf node at the end of $P$.
\label{lemma:flow_block}

\end{lemma}
\begin{proof}
Since $P$ does not contain a type \typeB edge, we can use Equation
\ref{eq:h_s_boolean}, which, as a reminder, states that
\begin{equation}
\label{eq:h_s_boolean_reprise}
    h_P(\replace{S})
    = v_l \prod_{
    \substack{e\in P\\
    e \text{ type \typeX}}}\!\!\!
    \mathbbm{1}(i_e\in S)\!\!\!
    \prod_{
    \substack{e\in P\\
    e \text{ type \typeZ}}}\!\!\!
     \mathbbm{1}(i_e\notin S).
\end{equation}

We need to prove the following two statements.
\begin{enumerate}
    \item $h_P(\replace{S_X})=v_l$. On the one hand, for every edge $e$ of type \typeX in $P$, by definition $i_e\in S_X$. On the other hand, for every edge $e$ of type \typeZ in $P$, by definition $i_e\in S_Z$, so $i_e\notin S_X$ since we assume that $S_X\cap S_Z=\emptyset$. Therefore it follows that $h_P(\replace{S_X})=v_l$ using Equation~\ref{eq:h_s_boolean_reprise}.
    \item If $S\subseteq S_{XZ}$ and $S\neq S_X$, then $h_P(\replace{S})=0$. Assume that $S\subseteq S_{XZ}$ and $S\neq S_X$. Since, $S\neq S_X$, at least of the
    two following cases must occur, (a) there exists $j\in S_X\setminus S$ or (b) there exists $j\in S \setminus S_X$.
    \begin{enumerate}
\item Let $j\in S_X\setminus S$. We can find $e\in P$ of type \typeX such that $i_e=j$. We then have $\mathbbm{1}(i_e\in S)=0$.
\item Let $j\in S\setminus S_X$. Since $S\subseteq S_{XZ}$, it follows that $j\in S_Z$. We can find $e\in P$ of type \typeZ such that $i_e=j$. Then as $j\in S$, we have $\mathbbm{1}(i_e\notin S)=0$.
\end{enumerate}
Given that either case (a) or case (b) must occur, 
using Equation~\ref{eq:h_s_boolean_reprise} we see 
that $h_P(\replace{S})=0$, as desired.
\end{enumerate}

\end{proof}

We now arrive at the main Theorem which highlights how Interventional TreeSHAP computes
Shapley values while avoiding the exponential cost w.r.t the
number of input features. This Theorem is the culmination of all previous results : \textbf{Lemma \ref{lemma:dummy}}, \textbf{Corollary \ref{corollary:shap_avoid_D} \& \ref{corollary:shap_disjoint}} and \textbf{Lemma \ref{lemma:SAB_dummy} \& \ref{lemma:flow_block}}.

\begin{theorem}[Complexity Reduction]
If $P$ contains no type \typeB edges and the sets $S_X$ and $S_Z$ are disjoint, then all features that are not in $S_{XZ}$ are dummies and we get
\begin{equation}
    \phi_i(h_P, \bm{x},\bm{z}) = \sum_{S\subseteq S_{XZ}\setminus \{i\}}W(|S|, |S_{XZ}|)\big(h_P(\,\replace{S\cup\{i\}}\,)-
    h_P(\,\replace{S}\,)\big)\quad i\in S_{XZ}.
    \label{eq:phi_final}
\end{equation}

The exponential cost $\mathcal{O}(2^{|S_{XZ}|})$ of computing these terms reduces to $\mathcal{O}(1)$ following
\begin{align}
    i \notin S_{XZ} &\Rightarrow \phi_i(h_P, \bm{x},\bm{z}) = 0  \label{final_dumb}\\ 
    i \in S_X &\Rightarrow \phi_i(h_P, \bm{x},\bm{z}) = W(|S_X|-1, |S_{XZ}|)v_l \label{final_S_X} \\
    i \in S_Z &\Rightarrow \phi_i(h_P, \bm{x},\bm{z}) = -W(|S_X|, |S_{XZ}|)v_l, \label{final_S_Z}
\end{align}
given that the coefficients $W$
where computed and stored in advance.
\label{thm:complexity_reduction}
\end{theorem}

\begin{proof}
Given \textbf{Lemma \ref{lemma:SAB_dummy}}, we know that all features not in $S_{XZ}$ are dummies. Hence applying \textbf{Lemma \ref{lemma:dummy}} yields Equation \ref{eq:phi_final} and \ref{final_dumb}. 
We now prove \ref{final_S_X} and \ref{final_S_Z} separately.

Suppose $i\in S_X$, then according to \textbf{Lemma \ref{lemma:flow_block}}
the only coalition $S\subseteq S_{XZ}\setminus \{i\}$ for which $h_P(\,\replace{S\cup\{i\}}\,)-
h_P(\,\replace{S}\,)$ is non-null is when $S=S_X\setminus\{i\}$. Indeed, when $S=S_X\setminus\{i\}$
\begin{align*}
h_P(\,\replace{S\cup\{i\}}\,)-
h_P(\,\replace{S}\,)&=h_P(\,\replace{S_X}\,)-
h_P(\,\replace{S_X\setminus\{i\}}\,)\\
&=v_l-0=v_l.
\end{align*} 
So, we get a Shapley value $W(|S|, |S_{XZ}|)v_l =W(|S_X\setminus\{i\}|, |S_{XZ}|)v_l=W(|S_X|-1, |S_{XZ}|)v_l$.

Now suppose $i\in S_Z$, according to \textbf{Lemma \ref{lemma:flow_block}}
the only coalition $S\subseteq S_{XZ}\setminus \{i\}$ for which $h_P(\,\replace{S\cup\{i\}}\,)-
h_P(\,\replace{S}\,)$ is non-null is when $S=S_X$ Indeed, 
when $S=S_X$
\begin{align*}
h_P(\,\replace{S\cup\{i\}}\,)-
h_P(\,\replace{S}\,)&=
h_P(\,\replace{S_X\cup \{i\}}\,)-
h_P(\,\replace{S_X}\,)\\
&=0-v_l=-v_l.
\end{align*} 
We therefore get a Shapley value $W(|S|, |S_{XZ}|)v_l =-W(|S_X|, |S_{XZ}|)v_l$.

\end{proof}

The complexity reduction Theorem is the main reason Interventional TreeSHAP works so efficiently. Indeed, the exponential complexity of Shapley values w.r.t the number of input features used to severely limit their application to high-dimensional machine learning tasks. Interventional TreeSHAP gets rid of this exponential complexity by noticing that decision stumps $h_P$ are \texttt{AND} functions. This drastically reduces the number of coalitions $S\subseteq S_{XZ}\setminus\{i\}$ to which adding player $i$ changes the output of the model.

\subsection{Efficient Tree Traversal}

In this section, we leverage theoretical results from the previous section to improve the dynamic tree traversal
in \textbf{Algorithm \ref{alg:simplest}}.
First, by virtue of \textbf{Theorem \ref{thm:complexity_reduction}}, given the path $P$ between the root and the current node in the traversal, we only need to keep track of the sets $S_X$ and $S_Z$.
Indeed once we reach a leaf (and $P$ becomes a maximal path), we will 
only need these two sets of features
to compute the Shapley values. 

Now that we know the two data structures to track
during the tree traversal, we are left with optimizing the tree traversal itself.
\textbf{Corollary \ref{corollary:shap_avoid_D} \& \ref{corollary:shap_disjoint}} give us sufficient conditions
for when the contributions of line 3 of \textbf{Algorithm \ref{alg:simplest}} are null. 
Hence, we must design the tree traversal
in a way that avoids reaching maximal paths $P$
whose contributions are guaranteed to be zero.

Firstly, from \textbf{Corollary \ref{corollary:shap_avoid_D}}, 
we should avoid going down type \typeB edges during 
the traversal. This is because any maximal path $P$ that 
follows this edge will contain a type $\typeB$ edge and 
hence have null Shapley values. Therefore, if during the 
tree traversal we encounter a split where both $\bm{x}$ and 
$\bm{z}$ flow through the same edge, we only need to follow 
that edge of type \typeF and the sets $S_X$ and $S_Z$ remain 
unchanged, see Figure \ref{fig:scenario_1_2}(a). We call 
this the \textbf{Scenario 1}.

Secondly, \textbf{Corollary \ref{corollary:shap_disjoint}} 
informs us that, as we traverse the tree, we must only 
follow paths where the sets $S_X$ and $S_Z$ are disjoint. 
If we follow paths
where the sets are not disjoint, then any maximal path we 
reach will have zero Shapley values resulting in wasted 
computations. Therefore, if a node $n$ is encountered such 
that $\bm{x}$ and $\bm{z}$ go different ways, and the 
associated feature $i_n$ is already in $S_{XZ}$, then
\begin{enumerate}
    \item if $i_n\in S_X$, go down the same direction as $\bm{x}$
    \item if $i_n\in S_Z$, go down the same direction as $\bm{z}$.
\end{enumerate}
This is necessary to keep the sets $S_X$ and $S_Z$ disjoint. 
These two cases are illustrated in Figure 
\ref{fig:scenario_1_2}(b) and (c) and are referred to as 
\textbf{Scenario 2}.

\begin{figure}
    \centering
    \begin{subfigure}[b]{0.32\textwidth}
        \centering
        \resizebox{5.2cm}{!}{
        \begin{tikzpicture}
\tikzset{nodestyle/.style={scale=1.33,draw=black,shape=circle,fill=white!97!black}}
\tikzset{edgetyle/.style={-stealth, line width=1mm}}

\def\step{1.5};

\node[nodestyle] at (0,0) (1) {$n$};
\draw[edgetyle,opacity=0.2] (0,1.5*\step) node[anchor=south] {...} -- (1);
\draw[edgetyle] (1) -- (-1*\step,-1*\step) node[anchor=north] {...};
\draw[edgetyle,color=blue] (1) -- (\step, -1*\step)node[anchor=north] {...};


\node at (0.5,3) {$(S_X, S_Z)$};
\node at (2.7,-1.5) {$(S_X, S_Z)$};

\draw[dashed,-stealth,line width=0.25mm]  
plot[smooth, tension=.7] coordinates {(0.5,2.75) (0.65,0) (2,-1.25)};
\end{tikzpicture}
        }
        \caption{\hypertarget{scenario_1}
        {\textbf{Scenario 1}}: Ignore the type 
        \typeB edge and follow the type \typeF 
        edge without storing $i_n$.}
    \end{subfigure}
    \hfill
    \begin{subfigure}[b]{0.32\textwidth}
        \centering
        \resizebox{5.2cm}{!}{
        \begin{tikzpicture}
\tikzset{nodestyle/.style={scale=1.33,draw=black,shape=circle,fill=white!97!black}}
\tikzset{edgetyle/.style={-stealth, line width=1mm}}

\def\step{1.5};

\node[nodestyle] at (0,0) (1) {$n$};
\draw[edgetyle,opacity=0.2] (0,1.5*\step) node[anchor=south] {...} -- (1);
\draw[edgetyle, color=red] (1) -- (-1*\step,-1*\step) node[anchor=north] {...};
\draw[edgetyle,color=green!70!black] (1) -- (\step, -1*\step)node[anchor=north] {...};


\node at (0.5,3) {$(S_X, S_Z)$};
\node at (2.7,-1.5) {$(S_X, S_Z)$};

\draw[dashed,-stealth,line width=0.25mm]  
plot[smooth, tension=.7] coordinates {(0.5,2.75) (0.65,0) (2,-1.25)};
\node at (1.5,0.5) {$i_n\in S_X$};
\end{tikzpicture}
        }
        \caption{\hypertarget{scenario_2}
        {\textbf{Scenario 2}}: Feature seen before 
        in $S_X$ so go down the direction of $\bm{x}$.}
    \end{subfigure}
    \hfill
    \begin{subfigure}[b]{0.32\textwidth}
        \centering
        \resizebox{5.2cm}{!}{
        \begin{tikzpicture}
\tikzset{nodestyle/.style={scale=1.33,draw=black,shape=circle,fill=white!97!black}}
\tikzset{edgetyle/.style={-stealth, line width=1mm}}

\def\step{1.5};

\node[nodestyle] at (0,0) (1) {$n$};
\draw[edgetyle,opacity=0.2] (0,1.5*\step) node[anchor=south] {...} -- (1);
\draw[edgetyle, color=red] (1) -- (-1*\step,-1*\step) node[anchor=north] {...};
\draw[edgetyle,color=green!70!black] (1) -- (\step, -1*\step)node[anchor=north] {...};


\node at (-0.5,3) {$(S_X, S_Z)$};
\node at (-2.7,-1.5) {$(S_X, S_Z)$};

\draw[dashed,-stealth,line width=0.25mm]  
plot[smooth, tension=.7] coordinates {(-0.5,2.75) (-0.65,0) (-2,-1.25)};
\node at (-1.5,0.5) {$i_n\in S_Z$};
\end{tikzpicture}
        }
        \caption{\textbf{Scenario 2}: 
        Feature seen before in $S_Z$ so go 
        down the direction of $\bm{z}$.}
    \end{subfigure}
    \caption{}
    \label{fig:scenario_1_2}
\end{figure}

\begin{figure}[t]
    \centering
    {\begin{tikzpicture}
\tikzset{nodestyle/.style={scale=1.33,draw=black,shape=circle,fill=white!97!black}}
\tikzset{edgetyle/.style={-stealth, line width=1mm}}

\def\step{1.5};

\node[nodestyle] at (0,0) (1) {$n$};
\draw[edgetyle,opacity=0.2] (0,1.5*\step) node[anchor=south] {...} -- (1);
\draw[edgetyle, color=red] (1) -- (-1*\step,-1*\step) node[anchor=north] {...};
\draw[edgetyle,color=green!70!black] (1) -- (\step, -1*\step)node[anchor=north] {...};


\node at (0,3) {$(S_X, S_Z)$};
\node at (-3.5,-1.5) {$(S_X, S_Z\cup \{i_n\})$};
\node at (3.5,-1.5) {$(S_X\cup \{i_n\}, S_Z)$};
\node at (1.6,0.5) {$i_n\notin S_X$};
\node at (-1.6,0.5) {$i_n\notin S_Z$};

\draw[dashed,-stealth,line width=0.25mm]  
plot[smooth, tension=.7] coordinates {(-0.5,2.75) (-0.65,0) (-2,-1.25)};
\draw[dashed,-stealth,line width=0.25mm]  
plot[smooth, tension=.7] coordinates {(0.5,2.75) (0.65,0) (2,-1.25)};

\end{tikzpicture}}
    \caption{\textbf{Scenario 3}: 
    Feature never seen 
    before in $S_X$ and $S_Z$ so go down both 
    directions and update $S_X$ and $S_Z$ 
    accordingly.}    
    \label{fig:scenario_3}
\end{figure}
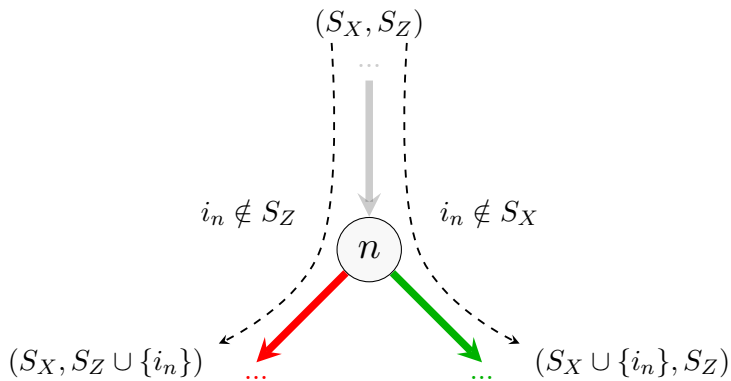

The final scenario that can occur when reaching a node is when $\bm{x}$ and $\bm{z}$ do not go the same way and the feature $i_n$ is not already in $S_{XZ}$. In that case. we must visit both branches and update the sets $S_X, S_Z$ accordingly, see Figure \ref{fig:scenario_3}. We refer to this as the \textbf{Scenario 3}.
We have now identified every scenario that can occur during the tree traversal, as well as the two data structures $S_X, S_Z$ that we must store during the exploration.
\textbf{Algorithm \ref{alg:tree_shap}} presents the final pseudo-code for Interventional TreeSHAP based on \textbf{Theorem \ref{thm:complexity_reduction}} and
our insights on how to efficiently traverse the tree.
The complexity of this algorithm depends on the data structure used to represent the sets $S_X, S_Z$. Interventional TreeSHAP must iterate through elements of $S_X$ and $S_Z$ in line 3. Moreover, on lines 4, 12, and 13, we check whether or not an index $i_n$ belongs to one of these sets. If we encode the sets as binary vectors and keep track of their cardinality, the amount of computations per internal node is $\mathcal{O}(1)$ and is
$\mathcal{O}(d)$ for terminal nodes (because of the for-loop at line 3). Hence, letting $I$ be the set of internal nodes, we get a global complexity of $\mathcal{O}(|I| + |L|d)$. This is already a large improvement compared to the exponential complexity of classical Shapley values (assuming we have a tree whose number of nodes scales desirably w.r.t $d$).

Still, the current implementation of 
Interventional TreeSHAP leverages additional computation
improvements to reduce the overall complexity to $\mathcal{O}(|I|) +\mathcal{O}(|L|)=\mathcal{O}(|N|)$. The intuition
is that we can avoid the for-loop at line 3 by propagating the contributions at lines
5 and 7 up in the tree. However, we have found that this additional optimization does not
generalize to other game theory indices, such as the Shapley-Taylor index presented
in the coming section. Hence, we prefer to see \textbf{Algorithm \ref{alg:tree_shap}}
as the most general/versatile formulation of Interventional TreeSHAP.

\begin{algorithm}[h]
\caption{Interventional Tree SHAP}
\begin{algorithmic}[1]
\Procedure{Recurse}{$n$, $S_X$, $S_Z$}
    \If{$n\in L$}
        \For{$i\in S_X\cup S_Z$}\Comment{cf. Theorem \ref{thm:complexity_reduction}}
            \If{$i \in S_X$}
                \State $\phi_i \pluseq W(|S_X|-1, |S_{XZ}|)v_n$;
            \Else
                \State $\phi_i \minuseq W(|S_X|, |S_{XZ}|)v_n$;
            \EndIf
        \EndFor
    \ElsIf{$\bm{x}_\text{child}=\bm{z}_\text{child}$}\Comment{\hyperlink{scenario_1}{Scenario 1}}
        \State \Return \Call{Recurse}{$\bm{x}_\text{child}, S_X, S_Z$};
    \ElsIf{$i_n\in S_X\cup S_Z$}\Comment{\hyperlink{scenario_2}{Scenario 2}}
        \If{$i_n \in S_X$}
            \State \Return \Call{Recurse}{$\bm{x}_\text{child}, S_X, S_Z$};
        \Else
            \State \Return \Call{Recurse}{$\bm{z}_\text{child}, S_X, S_Z$};
        \EndIf
    \Else\Comment{Scenario 3}
        \State \Call{Recurse}{$\bm{x}_\text{child}, S_X\cup\{i_n\}, S_Z$};
        \State \Call{Recurse}{$\bm{z}_\text{child}, S_X, S_Z\cup\{i_n\}$};
    \EndIf
\EndProcedure
\State $\bm{\phi}=\bm{0}$;
\State \Call{Recurse}{0, $S_X=\emptyset$, $S_Z=\emptyset$};
\State \Return $\bm{\phi}$;
\end{algorithmic}
\label{alg:tree_shap}
\end{algorithm}

\newpage
\textcolor{white}
\newpage
\section{Shapley-Taylor Indices}\label{sec:shap_taylor}

One of the purposes of deriving a proof for Interventional TreeSHAP is that
it can be leveraged to discover efficient algorithms for computing
other Game Theory indices.
In this section, we employ previous results to prove how to compute Shapley-Taylor Indices efficiently. 

\subsection{Game Theory}

\begin{definition}[Shapley Taylor Indices \citep{sundararajan2020shapley}]
Given a set of $d$ players $[d]:=\{1,2,\ldots,d\}$ and a cooperative game $\nu:2^{[d]}\rightarrow \R$, the Shapley-Taylor Indices are defined as
\begin{equation}
    \Phi_{ij}(\nu) = 
        \begin{cases}
        \nu(\{i\})-\nu(\emptyset) & \text{if  } i=j\quad\,\,\,(\textbf{Main Effect})\\
        \mathlarger{\sum}_{S\subseteq [d]\setminus\{i, j\}} {W(|S|, d)}\nabla_{ij}(S) & \text{otherwise}\,\,\,\,(\textbf{Interactions}).
    \end{cases}
\end{equation}
with
\begin{equation}
    \nabla_{ij}(S) = \nu(S\cup \{i,j\}) - \nu(S\cup\{j\})
    - \big[\nu(S\cup \{i\}) - \nu(S)\big].
\end{equation}
\label{def:shapley_taylor}
\end{definition}
The Shapley-Taylor Indices were invented as a way to provide interaction indices that respect the same additive property as the original Shapley values
\begin{equation}
    \sum_{i=1}^d\sum_{j=1}^d \Phi_{ij}(\nu) = \nu([d]) - \nu(\emptyset).
\end{equation}

As with classical Shapley values, the Shapley-Taylor indices
assign no credit to dummy players. Letting $D$ be the set of dummy
players, we have
\begin{equation}
    i\in D \,\,\text{ or }\,\,j\in D\Rightarrow \Phi_{ij}(\game)=0.
\end{equation}
Like previously, the Shapley-Taylor indices of non-dummy players follow
\textbf{Definition \ref{def:shapley_taylor}}, but where the set of
all players $[d]$ is replaced by all non-dummy players $D^C$.
\begin{lemma}[Dummy Reduction for Shapley-Taylor]
Let $D\subseteq [d]$ be the set of all dummy players of the game $\nu$, and $D^C=[d]\setminus D$ be the set of non-dummy players, then
\begin{equation}
    \Phi_{ij}(\nu) = 
    \begin{cases}
       0 &\quad\text{if}\,\, i\in D \text{ or } j\in D\\
       \nu(\{i\})-\nu(\emptyset) &\quad\text{else if}\,\, i=j\\
       \sum_{S\subseteq D^C \setminus\{i,j\}} W(|S|, |D^C|)\nabla_{ij}(S) &\quad\text{otherwise.}
     \end{cases}
\end{equation}
\label{lemma:dummy_shap_taylor}
\end{lemma}
\begin{proof}
The proof is mutatis mutandis like the proof of \textbf{Lemma \ref{lemma:dummy}}.
\end{proof}

We consider in the sequel the Shapley-Taylor indices associated with the Baseline Interventional game $\nu_{\hxz}$ (cf. \textbf{Definition \ref{def:explain_game}}) for a model $h$, an instance $\bm{x}$ and a baseline instance $\bm{z}$. We simply write $\bm{\Phi}(h, \bm{x}, \bm{z})$ in place of $\bm{\Phi}(\nu_{\hxz})$ for these indices.
\subsection{Shapley-Taylor Indices Computation}

Since the Shapley-Taylor Indices are also linear w.r.t coalitional
games, computing these indices for a forest of decision trees reduces to the computation of these indices for a decision stump $h_P$ associated with a maximal path $P$. Just like with Shapley values, we can use \textbf{Algorithm \ref{alg:simplest}} as the skeleton of our algorithm. Fixing a decision tree $h$, instances $\bm{x}$ and $\bm{z}$, as well as a maximal path $P$ in $h$, we focus on the computation of the Shapley-Taylor Indices $\bm{\Phi}(h_P, \bm{x}, \bm{z})$. 

First we note that \textbf{Lemma \ref{lemma:avoid_D} \& \ref{lemma:disjoint}}
lead to similar corollaries about Shapley-Taylor Indices. 

\begin{corollary}
If $P$ contains an edge of type \typeB,
then $\bm{\Phi}(h_P, \bm{x}, \bm{z}) = \bm{0}$.
\label{corollary:shap_taylor_avoid_D}
\end{corollary}
\begin{proof}
Follows from \textbf{Lemma \ref{lemma:avoid_D}} and the linearity of Shapley-Taylor indices.
\end{proof}

\begin{corollary}
    If $S_X\cap S_Z \neq \emptyset$, then
    $\bm{\Phi}(h_P, \bm{x},\bm{z}) = \bm{0}.$
    \label{corollary:shap_taylor_disjoint}
\end{corollary}
\begin{proof}
By \textbf{Lemma \ref{lemma:disjoint}} and Linearity of the Shapley-Taylor values.
\end{proof}

Once we reach a leaf we are meant to compute the Shapley-Taylor Indices of the decision stump $h_P$ associated with the current maximal path $P$. We again use \textbf{Lemma \ref{lemma:flow_block}} to reduce the complexity.

\begin{theorem}[Complexity Reduction for Shapley-Taylor]
If $P$ contains no type \typeB edges and the sets $S_X$ and $S_Z$ are disjoint, then all features that are not in $S_{XZ}$ are dummies and we get
\begin{align}
    \Phi_{ii}(h_P, \bm{x},\bm{z}) &= h_P(\replace{\{i\}}) - h_P(\replace{\emptyset})&& i\in S_{XZ},  \label{eq:phi_ii_final_taylor} \\
    \Phi_{ij}(h_P, \bm{x},\bm{z}) &= \sum_{S\subseteq S_{XZ}\setminus \{i,j\}}W(|S|, |S_{XZ}|)\nabla_{ij}(S) && i,j\in S_{XZ}, i\neq j    \label{eq:int_phi_final}
\end{align}
with
\begin{equation}
    \nabla_{ij}(S):=h_P(\,\replace{S\cup\{i,j\}}\,)-
    h_P(\,\replace{S\cup\{j\}}\,) - \big[h_P(\,\replace{S\cup\{i\}}\,)-
    h_P(\,\replace{S}\,)\big]
\end{equation}

The exponential cost $\mathcal{O}(2^{|S_{XZ}|})$ of computing $\Phi_{ij}$ reduces to $\mathcal{O}(1)$ following
\begin{align}
    S_X=\{i\} &\Rightarrow \Phi_{ii}(h_P, \bm{x},\bm{z}) = v_l  \label{eq:FinalTaylor1}\\
    i \in S_Z \text{ and } S_X=\emptyset &\Rightarrow \Phi_{ii}(h_P, \bm{x},\bm{z}) = -v_l  \label{eq:FinalTaylor2}\\
    i \neq j \text{ and } i,j \in S_X &\Rightarrow \Phi_{ij}(h_P, \bm{x},\bm{z}) = W(|S_X|-2, |S_{XZ}|)v_l  \label{eq:FinalTaylor3}\\
    i \neq j \text{ and } i,j \in S_Z &\Rightarrow \Phi_{ij}(h_P, \bm{x},\bm{z}) = W(|S_X|, |S_{XZ}|)v_l  \label{eq:FinalTaylor4}
    \\
    i \neq j \text{ and } i \in S_X, j\in S_Z &\Rightarrow \Phi_{ij}(h_P, \bm{x},\bm{z}) = -W(|S_X|-1, |S_{XZ}|)v_l  \label{eq:FinalTaylor5}\\
    \text{else} &\quad\,\,\, \Phi_{ij}(h_P, \bm{x},\bm{z}) = 0,  \label{eq:FinalTaylor6}
\end{align}

given that the coefficients $W$
are computed and stored in advance.
\label{thm:taylor_complexity_reduction}
\end{theorem}
\begin{proof}
Equations \eqref{eq:phi_ii_final_taylor} and \eqref{eq:int_phi_final} are a direct consequence of \textbf{Lemma \ref{lemma:dummy_shap_taylor}}. 

We will first tackle the coefficients $\Phi_{ii}(h_P, \bm{x}, \bm{z}) = h_P(\replace{\{i\}}) - h_P(\replace{\emptyset})$ with $i\in S_{XZ}$. There are only
two possible cases where this difference is non-null. 
According to \textbf{Lemma \ref{lemma:flow_block}}, when $S_X=\{i\}$, we have
$h_P(\replace{\{i\}}) - h_P(\replace{\emptyset}) = h_P(\replace{S_X}) - h_P(\replace{\emptyset}) = v_l - 0=v_l$. When $S_X=\emptyset$ and
$i\in S_Z$,
we have $h_P(\replace{\{i\}}) - h_P(\replace{\emptyset}) =h_P(\replace{S_X\cup\{i\}}) - h_P(\replace{S_X})= 0 - v_l=-v_l$. We have thus proven Equations \eqref{eq:FinalTaylor1} and \eqref{eq:FinalTaylor2}.

Secondly, we tackle the non-diagonal elements $i\neq j$. To reduce the exponential complexity of computing these terms, we must identify the coalitions $S\subseteq S_{XZ}\setminus \{i,j\}$ for which
$$
\nabla_{ij}(S):=h_P(\,\replace{S\cup\{i,j\}}\,)-
    h_P(\,\replace{S\cup\{j\}}\,) - \big[h_P(\,\replace{S\cup\{i\}}\,)-
    h_P(\,\replace{S}\,)\big]
$$
is non-null. Recall that by \textbf{Lemma \ref{lemma:flow_block}}, $h_P(\replace{S})=0$ for all $S\subseteq S_{XZ}$ except for $S=S_X$. So for the contribution of a coalition $S$ to be non-zero, one of $S$, $S\cup\{i\}$, $S\cup \{j\}$ or $S\cup \{i,j\}$ must be equal to $S_X$. We distinguish four cases.

If $i,j\in S_X$, the only coalition $S \subseteq S_{XZ}\setminus \{i,j\}$ for which $\nabla_{ij}(S)$ is non zero is given by $S=S_X\setminus \{i,j\}$ and $\nabla_{ij}(S)= v_l - 0 - [0 - 0]=v_l$. Hence in this case, $\Phi_{ij}(h_P, \bm{x}, \bm{z})=W(|S|, |S_{XZ}|)v_l=W(|S_X|-2, |S_{XZ}|)v_l$.


If $i,j\in S_Z$, and so $i,j\notin S_X$, the only coalition with a non-zero contribution is $S=S_X$ and $\nabla_{ij}(S):=0-0 - \big[0 - v_l\big]=v_l$. It follows that $\Phi_{ij}(h_P, \bm{x}, \bm{z})=W(|S|, |S_{XZ}|)v_l=W(|S_X|, |S_{XZ}|)v_l$.


If $i\in S_X$ and $j\in S_Z$, the only coalition that contributes to
the Shapley-Taylor indices is given by $S=S_X\setminus\{i\}$, which leads to 
$\nabla_{ij}(S):=0-0 - \big[v_l - 0\big]=-v_l$.
We have $\Phi_{ij}(h_P, \bm{x}, \bm{z})=-W(|S|, |S_{XZ}|)v_l=-W(|S_X|-1, |S_{XZ}|)v_l$.


Finally, the case where $j\in S_X$ and $i\in S_Z$ is symmetric and leads to the same formula as the previous case. This proves Equations \eqref{eq:FinalTaylor3}, \eqref{eq:FinalTaylor4} and \eqref{eq:FinalTaylor5} and concludes the proof.
\end{proof}

We present the procedure for efficient computations of Shapley-Taylor indices in \textbf{Algorithm
\ref{alg:taylor_tree_shap}}. Note that our results on Interventional TreeSHAP easily apply to these new indices, which we view as a testament to the great value of our mathematical proof for TreeSHAP.
\begin{algorithm}
\caption{Interventional Taylor-TreeSHAP}
\begin{algorithmic}[1]
\Procedure{Recurse}{$n$, $S_X$, $S_Z$}
    \If{$n\in L$}\Comment{cf. Theorem \ref{thm:taylor_complexity_reduction}}
        \For{$i\in S_X\cup S_Z$}
            \For{$j\in S_X\cup S_Z$}
                \If{$i=j$}
                    \If{$S_X=\{i\}$}
                        \State $\Phi_{ii} \pluseq v_n$;
                    \ElsIf{$S_X=\emptyset$}
                        \State $\Phi_{ii} \minuseq v_n$;
                    \EndIf
                \Else
                    \If{$i,j \in S_X$}
                        \State $\Phi_{ij} \pluseq W(|S_X|-2, |S_{XZ}|)v_n$;
                    \ElsIf {$i,j \in S_Z$}
                        \State $\Phi_{ij} \pluseq W(|S_X|, |S_{XZ}|)v_n$;
                    \Else
                        \State $\Phi_{ij} \minuseq W(|S_X|-1, |S_{XZ}|)v_n$;
                    \EndIf
                \EndIf
            \EndFor
        \EndFor
     \ElsIf{$\bm{x}_\text{child}=\bm{z}_\text{child}$}\Comment{\hyperlink{scenario_1}{Scenario 1}}
        \State \Return \Call{Recurse}{$\bm{x}_\text{child}, S_X, S_Z$};
    \ElsIf{$i_n\in S_X\cup S_Z$}\Comment{\hyperlink{scenario_2}{Scenario 2}}
        \If{$i_n \in S_X$}
            \State \Return \Call{Recurse}{$\bm{x}_\text{child}, S_X, S_Z$};
        \Else
            \State \Return \Call{Recurse}{$\bm{z}_\text{child}, S_X, S_Z$};
        \EndIf
    \Else\Comment{Scenario 3}
        \State \Call{Recurse}{$\bm{x}_\text{child}, S_X\cup\{i_n\}, S_Z$};
        \State \Call{Recurse}{$\bm{z}_\text{child}, S_X, S_Z\cup\{i_n\}$};
    \EndIf
\EndProcedure
\State $\bm{\Phi}=\texttt{zeros}(d,d)$;
\State \Call{Recurse}{0, $S_X=\emptyset$, $S_Z=\emptyset$};
\State \Return $\bm{\Phi}$;
\end{algorithmic}
\label{alg:taylor_tree_shap}
\end{algorithm}

\newpage

\section{Extension to Partitions of Features}\label{sec:embed}

In certain settings, one may be interested in characterizing the effect of a set of features on the model's output, instead of considering their individual effects. This can be achieved by treating a set of features as a single super-feature. Assuming $d'$ features are fed to the model, let us consider that we partition them in $d$ groups: $[d']=\bigcup \{P_i |i\in[d]\}\,$. This yields a new attribution problem where we wish to associate a real number to each group $i\in[d]$ that represents the contribution of the features in $P_i$, when considered as a group, toward the model output. An important application of group attributions is when one-hot encoded categorical features are fed to the Decision Trees. We will discuss this application later in the section.

A partition of features $[d']$ can be conveniently described with an indexing function 
$\I:[d']\to[d]$ that associates each feature $j\in [d']$ to its group index
$\I(j)$. Note that the pre-image map $\I^{-1}:2^{[d]}\to 2^{[d']}$ allows to recover the groups of features using the partition notation: $\I^{-1}(\{i\})=P_i$, $i\in[d]$.
The partition of the $d'$ features into $d$ groups naturally induces a way to associate to any game $\game$ on $[d']$ a game $\game^\mathcal{I}$ on $[d]$ given by $\game^\mathcal{I}(S)=\game(\mathcal{I}^{-1}(S))$ for all $S\subseteq [d]$. In our specific case, given inputs $\bm{x}$ and $\bm{z}$ for a model $h$, we have
\begin{equation}
    \game^\I_{\hxz}(S):= \game_{\hxz}(\I^{-1}(S))
    =h(\,\replace{\I^{-1}(S)}\,).
\end{equation}
See Figure \ref{fig:perm_groups} for an illustration of how the replace function $\replace{\I^{-1}(S)}$ is employed in $\game^\I$. We note that features within the same group are always replaced simultaneously and so Shapley values for the game $\nu^\I_{\hxz}$ represent the joint contributions of each group towards the gap $h(\bm{x})-h(\bm{z})$.
Computing Shapley values for the game $\nu^\I_{\hxz}$ for a decision Tree $h$ cannot be done using TreeSHAP, but we now show how to leverage our previous results to adapt TreeSHAP in this new setting.

Firstly, by linearity of the Shapley values we can again focus on a single maximal path $P$, and study the Shapley values of the game $\game^\I_{\hpxz}$. Henceforth, we freely use the concepts and notations introduced in the previous sections for the computation of the Shapley values for the game $\game_{\hpxz}$ on $[d']$ and we show how to use them to analyze the game $\game^\I_{\hpxz}$. For example, type \typeX, type \typeZ edges and the sets $S_X$, $S_Z$ are all defined as in \textbf{Section~\ref{sec:tree_shap}} with respect to the maximal path $h_P$ in a decision tree $h$ and the inputs $\bm{x}$ and $\bm{z}$.
We note that if $P$ contains a type \typeB edge, then \textbf{Lemma \ref{lemma:avoid_D}} ensures that the game $\game_{h_P,\bm{x},\bm{z}}$ is null and this implies that the game $\game^\I_{h_P,\bm{x},\bm{z}}$ is null too. The resulting Shapley values 
$\bm{\phi}(\game^\I_{h_P,\bm{x},\bm{z}})$ are therefore constantly zero. Hence, we can again ignore maximal paths with type \typeB edges. 
The next result is analogous to \textbf{Lemma \ref{lemma:disjoint}} but for the game $\game^\I$.

\begin{figure}[t]
    \hspace{2.5cm}
        \resizebox{8cm}{!}{\begin{tikzpicture}

    \def\y{0}
    \draw[step=0.5] (-1, \y) grid (3, \y+0.5);
    \node at (-1.5, \y+0.25) {$\bm{z}$};
    \node at (-0.75, \y+0.25) {$\textcolor{blue}{z_0}$};
    \node at (-0.25, \y+0.25) {$\textcolor{red}{z_1}$};
    \node at (0.25, \y+0.25) {$\textcolor{red}{z_2}$};
    \node at (0.75, \y+0.25) {$\textcolor{orange}{z_3}$};
    \node at (1.25, \y+0.25) {$\textcolor{orange}{z_4}$};
    \node at (1.75, \y+0.25) {$\textcolor{orange}{z_5}$};
    \node at (2.25, \y+0.25) {$\textcolor{orange}{z_6}$};
    \node at (2.75, \y+0.25) {$\green{z_7}$};

    \def\y{-1.5}
    \draw[step=0.5] (-1, \y) grid (3, \y+0.5);
    \node at (-1.5, \y+0.25) {$\bm{x}$};
    \node at (-0.75, \y+0.25) {$\textcolor{blue}{x_0}$};
    \node at (-0.25, \y+0.25) {$\textcolor{red}{x_1}$};
    \node at (0.25, \y+0.25) {$\textcolor{red}{x_2}$};
    \node at (0.75, \y+0.25) {$\textcolor{orange}{x_3}$};
    \node at (1.25, \y+0.25) {$\textcolor{orange}{x_4}$};
    \node at (1.75, \y+0.25) {$\textcolor{orange}{x_5}$};
    \node at (2.25, \y+0.25) {$\textcolor{orange}{x_6}$};
    \node at (2.75, \y+0.25) {$\green{x_7}$};

    \def\y{-3}
    \draw[step=0.5] (-1, \y) grid (3, \y+0.5);
    \node at (-2.35, \y+0.2) {$\replace{\I^{-1}(\,\{\textcolor{blue}{0},\textcolor{orange}{2}\}\,)}$};
    \node at (-0.75, \y+0.25) {$\textcolor{blue}{x_0}$};
    \node at (-0.25, \y+0.25) {$\textcolor{red}{z_1}$};
    \node at (0.25, \y+0.25) {$\textcolor{red}{z_2}$};
    \node at (0.75, \y+0.25) {$\textcolor{orange}{x_3}$};
    \node at (1.25, \y+0.25) {$\textcolor{orange}{x_4}$};
    \node at (1.75, \y+0.25) {$\textcolor{orange}{x_5}$};
    \node at (2.25, \y+0.25) {$\textcolor{orange}{x_6}$};
    \node at (2.75, \y+0.25) {$\green{z_7}$};

	\draw[-stealth] (-0.75,0) -- (-0.75,-1);
	\draw[-stealth] (1.5,0) -- (1.5,-1);
\end{tikzpicture}}
        \caption{The replace function applied to partitions of $d'=8$ features into $d=4$ groups indicated by colors. Importantly, all features in the same group are replaced simultaneously.} 
    \label{fig:perm_groups}
\end{figure}

\begin{lemma}
    If $ \I(S_X) \cap \I(S_Z)\neq \emptyset $, then 
    \begin{equation}
        \forall\, S\subseteq [d] \quad \game^\I_{\hpxz}(S) = 0.
    \end{equation}    
    \label{lemma:disjoint_part}
\end{lemma}

\begin{proof}
Since $ \I(S_X) \cap \I(S_Z)\neq \emptyset $, we can find an index $j\in [d]$ that belongs to both $\I(S_X)$ and $\I(S_Z)$. By definition of $S_X$ and $S_Z$, this means that we can find in $P$ an edge $e$ of \typeX and an edge $e'$ of type \typeZ such that $\I(i_{e})=j=\I(i_{e'})$. Now using Equation~\ref{eq:h_s_boolean}, for all $S\subseteq [d]$ we obtain:
\begin{align*}
\game^\I_{\hpxz}(S)=\game_{\hpxz}(\I^{-1}(S))&\propto \mathbbm{1}(i_e\in \I^{-1}(S))\mathbbm{1}(i_{e'}\notin \I^{-1}(S))
\\
&=\mathbbm{1}(\I(i_e)\in S)\mathbbm{1}(\I(i_{e'})\notin S)\\
&=\mathbbm{1}(j\in S)\mathbbm{1}(j\notin S)=0.
\end{align*}
\end{proof}


We have just identified two necessary conditions for the Shapley values of $\game^\I$ to be null: if $P$ contains a type \typeB edge or if the sets $\I(S_X)$ and $\I(S_Z)$ are not disjoint. We will henceforth assume that none of these conditions apply to $P$ so that the Shapley values are potentially non-null.
We now wish to understand what players of $\game^\I$ are dummies. 
The following Lemma establishes a relationship between dummies of $\game$ and dummies of $\game^\I$.


\begin{lemma}
If $\I^{-1}(\{i\})$ is a set of dummies of $\game$,
then $i$ is a dummy player of $\game^\I$.
This is true for any games $\nu$ and $\nu^\I$ and more specifically Baseline Interventional Games
$\nu_{\hpxz}$.
\label{lemma:dummy_partition}
\end{lemma}
\begin{proof}
First notice that if $D\subseteq [d']$ is a set of dummy players, then it easily follows by induction that for all $S'\subseteq [d']\setminus D$ we have $\game(S'\cup D)=\game(S')$. Now let $S\subseteq [d]\setminus \{i\}$ and observe that when $\I^{-1}(\{i\})$ is a set of dummy players for $\game$ we have:
\begin{align*}
\game^\I(S\cup \{i\}) &= \game\big(\,\I^{-1}(S\cup \{i\})\,\big)\\
&= \game\big(\,\I^{-1}(S)\cup \I^{-1}(\{i\})\,\big)\\
&= \game\big(\,\I^{-1}(S)\,\big)\\
&= \game^\I(S)
\end{align*}
\end{proof}
\vspace{-0.05cm}
We now present an analogue to \textbf{Lemma \ref{lemma:SAB_dummy}}.

\begin{lemma}
    If $i$ does not belong to $\I(S_{XZ})$, then it is a dummy player for the game $\game^\I_{h_P,\bm{x},\bm{z}}$.
    \label{lemma:SXZ_dummy_part}
\end{lemma}
\begin{proof}
If $i \notin \I(S_{XZ})$ then $\I^{-1}(\{i\})\cap S_{XZ}=\emptyset$. So
by \textbf{Lemma \ref{lemma:SAB_dummy}}, $\I^{-1}(\{i\})$ are dummies of $\game_{h_P,\bm{x},\bm{z}}$ which by \textbf{Lemma \ref{lemma:dummy_partition}} implies that $i$ is a dummy feature of $\game^\I_{h_P,\bm{x},\bm{z}}$.
\end{proof}

As only the features in $S_{XZ}$ are not dummies of $\nu_{h_P,\bm{x},\bm{z}}$, it follows that 
\[
\nu^\I_{h_P,\bm{x},\bm{z}}(S)=\nu_{h_P,\bm{x},\bm{z}}(\I^{-1}(S)\cap S_{XZ}) \quad \text{for all $S\subseteq[d]$.}
\]
We will use this fact in the sequel.
The following key Theorem shows how to efficiently compute the Shapley values of non-dummy players.


\begin{theorem}[Complexity Reduction for Partitions]
If $P$ contains no type \typeB edges and the sets $\I(S_X)$ and $\I(S_Z)$ are disjoint,
we get
\begin{align}
    i \notin \I(S_{XZ}) &\Rightarrow \phi_i^\I(h_P, \bm{x},\bm{z}) = 0  \label{final_dumb_part}\\ 
    i \in \I(S_X) &\Rightarrow \phi_i^\I(h_P, \bm{x},\bm{z}) = W(|\I(S_X)|-1, |\I(S_{XZ})|)v_l \label{final_S_X_part} \\
    i \in \I(S_Z) &\Rightarrow \phi^\I_i(h_P, \bm{x},\bm{z}) = -W(|\I(S_X)|, |\I(S_{XZ})|)v_l. \label{final_S_Z_part}
\end{align}
\label{thm:complexity_reduction_part}
\end{theorem}

\begin{proof}
The equation \ref{final_dumb_part} is a direct consequence of \textbf{Lemma \ref{lemma:SXZ_dummy_part}}.
We now prove \ref{final_S_X_part} and \ref{final_S_Z_part} separately.

Since all features in $[d]\setminus \I(S_{XZ})$ are dummies of $\nu^\I_{h_P,\bm{x},\bm{z}}$, we can employ \textbf{Lemma \ref{lemma:dummy}} to rewrite the Shapley value definition
\begin{equation}
    \phi_i(\game^\I)
    =\mathlarger{\sum}_{S\subseteq \I(S_{XZ})\setminus\{i\}} W(|S|, |\I(S_{XZ})|\,)\big(\game_{h_P,\bm{x},\bm{z}}^\I(S\cup \{i\}) - \game_{h_P,\bm{x},\bm{z}}^\I(S)\big)\qquad i\in\I(S_{XZ}).
    \label{eq:phi_part_reduced}
\end{equation}
This summation contains an exponential number of terms, and like previous Theorems, only one of these terms will be non-zero. To identify which term is non-zero, we will again rely on \textbf{Lemma \ref{lemma:flow_block}}.
To do so, we note that since $\I(S_X)$ and $\I(S_Z)$ are disjoint, the sets $S_X$ and $S_Z$ must also be disjoint. Hence by \textbf{Lemma \ref{lemma:flow_block}}, we have $\game_{h_P,\bm{x},\bm{z}}(S')=0$ for all $S'\subseteq S_{XZ}$ except for $S'=S_X$. Since as noted earlier $\game^\I_{h_P,\bm{x},\bm{z}}(S)=\game_{h_P,\bm{x},\bm{z}}(\I^{-1}(S)\cap S_{XZ})$ for all $S\subseteq [d]$, \textbf{Lemma \ref{lemma:flow_block}} ensures that we have the following: $\game^\I_{h_P,\bm{x},\bm{z}}(S)=0$ for all $S\subseteq \I(S_{XZ})$ except when $\I^{-1}(S)\cap S_{XZ}=S_X$, or equivalently $S=\I(S_X)$. 
Therefore the only way a difference $\game_{h_P,\bm{x},\bm{z}}^\I(S\cup \{i\}) - \game^\I_{h_P,\bm{x},\bm{z}}(S)$ can be non-null is when either $S$ or $S\cup\{i\}$ is equal to $\I(S_X)$.


Now, suppose $i\in \I(S_X)$. The only $S\subseteq \I(S_{XZ})\setminus\{i\}$ for which the difference $\game_{h_P,\bm{x},\bm{z}}^\I(S\cup \{i\}) - \game^\I_{h_P,\bm{x},\bm{z}}(S)$ is not zero is when $S=\I(S_X)\setminus \{i\}$. For this set $S$, we get a contribution
$W(|S|, |\I(S_{XZ})|)v_l =W(|\I(S_X)\setminus\{i\}|, |\I(S_{XZ})|)v_l=W(|\I(S_X)|-1, |\I(S_{XZ})|)v_l$ thus proving Equation \ref{final_S_X_part}.


Finally, suppose $i\in \I(S_Z)$ and so $i\notin \I(S_X)$. Then 
the only way the difference $\game_{h_P,\bm{x},\bm{z}}^\I(S\cup \{i\}) - \game_{h_P,\bm{x},\bm{z}}^\I(S)$ can be non-null
is when $S=\I(S_X)$. For this set $S$, we get a contribution
$-W(|S|, |\I(S_{XZ})|)v_l =-W(|\I(S_X)|, |\I(S_{XZ})|)v_l$ thus proving Equation \ref{final_S_Z_part}.
\end{proof}

We deduce from \textbf{Lemma \ref{lemma:disjoint_part}} and \textbf{Theorem \ref{thm:complexity_reduction_part}} that computing Shapley values for the game $\nu^\I$ can be achieved by substituting the sets $\I(S_X)$ and $\I(S_Z)$ to
$S_X$ and $S_Z$ in the original TreeSHAP algorithm, as shown in Algorithm \ref{alg:tree_shap_part}.
The simplicity with which we were able to  adapt TreeSHAP to partitions is a by-product of the theoretical results we obtained on TreeSHAP. We conclude by noting that if each feature is its own group (\textit{i.e.} $d'=d$ and $\I(i)=i$), Algorithm \ref{alg:tree_shap_part} falls back to Algorithm \ref{alg:tree_shap}.
The coming subsection describes an application of Partition-TreeSHAP.

\begin{algorithm}
\caption{Interventional Partition-TreeSHAP}
\begin{algorithmic}[1]
\Procedure{Recurse}{$n$, $\I(S_X)$, $\I(S_Z)$}
    \If{$n\in L$}
        \For{$i\in \I(S_X)\cup \I(S_Z)$}\Comment{cf. Theorem \ref{thm:complexity_reduction}}
            \If{$i \in \I(S_X)$}
                \State $\phi_i \pluseq W(|\I(S_X)|-1, |\I(S_{XZ})|)v_n$;
            \Else
                \State $\phi_i \minuseq W(|\I(S_X)|, |\I(S_{XZ})|)v_n$;
            \EndIf
        \EndFor
     \ElsIf{$\bm{x}_\text{child}=
     \bm{z}_\text{child}$}\Comment{\hyperlink{scenario_1}{Scenario 1}}
        \State \Return \Call{Recurse}{$\bm{x}_\text{child}, \I(S_X), \I(S_Z)$};
    \ElsIf{$\I(i_n)\in \I(S_X)\cup \I(S_Z)$}\Comment{\hyperlink{scenario_2}{Scenario 2}}
        \If{$\I(i_n) \in \I(S_X)$}
            \State \Return \Call{Recurse}{$\bm{x}_\text{child}, \I(S_X), \I(S_Z)$};
        \Else
            \State \Return \Call{Recurse}{$\bm{z}_\text{child}, \I(S_X), \I(S_Z)$};
        \EndIf
    \Else\Comment{Scenario 3}
        \State \Call{Recurse}{$\bm{x}_\text{child}, \,\I(S_X)\cup\I(\{i_n\}),\, \I(S_Z)$};
        \State \Call{Recurse}{$\bm{z}_\text{child},\, \I(S_X),\, \I(S_Z)\cup\I(\{i_n\})$};
    \EndIf
\EndProcedure
\State $\bm{\phi}=\bm{0}$;
\State \Call{Recurse}{0, $\I(S_X)=\emptyset$, $\I(S_Z)=\emptyset$};
\State \Return $\bm{\phi}$;
\end{algorithmic}
\label{alg:tree_shap_part}
\end{algorithm}

\subsection{Application to Feature Embedding}

Classic training procedures for Decision Trees partition the input space at each internal node $n$ via Boolean functions of the form $\mathbbm{1}(x_{i_n}\leq \gamma_n)$. However, this model architecture assumes that feature values can be ordered, which always holds for numerical data, but not necessarily for categorical features \textit{e.g.} $x_i\in [\text{Dog}, \text{Cat}, \text{Hamster}]$.
In such scenarios, the model architecture itself must be adapted \citep[Section 9.2.4]{hastie2009elements}, or the categorical features must be embedded in a metric space during pre-processing \citep{guo2016entity}. An efficient approach for computing Interventional Shapley values in such contexts depends on how categorical features are handled by the model. In this subsection, we show how Partition-TreeSHAP can compute Shapley values for categorical features when the latter approach is used. Namely, we consider here the case where the tree growth strategies are kept the same and categorical features are embedded before being fed to the model. This choice was made because \texttt{scikit-learn} Decision Trees do not yet natively support categorical features\footnote{\url{https://github.com/scikit-learn/scikit-learn/pull/12866}}.
Concretely, assume that each feature $i\in[d]$ takes value in a set $D_i$ which is assumed to be either $\R$ for numerical features or a finite set $[C]$ for a categorical feature taking $C$ possible values. Before learning a tree, the data is preprocessed by choosing for every feature $i\in[d]$ an embedding function $\xi_i:D_i\to \R^{d_i}$. Several types of embedding are possible.
\begin{itemize}
    \item The identity embedding $\xi_i(x_i) = x_i \in \R^1$ is used for numerical features where $D_i=\R$.
    \item The one-hot-encoding $\xi_i(x_i) = \bm{\delta}_{x_i}\in \R^C$
    when $x_i$ is categorical and takes $C$ possible values, i.e. $D_i=[C]$. The components of the vector $\bm{\delta}_j$ are all $0$ except for the $j$th component which is $1$.
    \item The entity embedding $\xi_i(x_i) = \mathbf{E}[:, i]\in \R^{d_i}$
    where the embedding matrix $\mathbf{E}\in\R^{d_i\times C}$ is learned \emph{a priori} via a Neural Network \citep{guo2016entity}.
\end{itemize}

\begin{figure}[t]
    \centering
    \resizebox{7cm}{!}{\begin{tikzpicture}
    \def\y{-1.5}

    \draw[step=0.5] (0, 0) grid (2, 0.5);
    \node at (-0.25, 0.25) {$\bm{x}$};
    \node at (0.25, 0.25) {$\textcolor{blue}{x_0}$};
    \node at (0.75, 0.25) {$\textcolor{red}{x_1}$};
    \node at (1.25, 0.25) {$\textcolor{orange}{x_2}$};
    \node at (1.75, 0.25) {$\green{x_3}$};

    \node at (-1.5, \y+0.25) {$\bm{\xi}(\bm{x})$};
    \draw[step=0.5] (-1, \y) grid (3, \y+0.5);
    \node at (-0.75, \y+0.25) {$\textcolor{blue}{x_0}$};
    \node at (-0.25, \y+0.25) {$\textcolor{red}{0}$};
    \node at (0.25, \y+0.25) {$\textcolor{red}{1}$};
    \node at (0.75, \y+0.25) {$\textcolor{orange}{0}$};
    \node at (1.25, \y+0.25) {$\textcolor{orange}{0}$};
    \node at (1.75, \y+0.25) {$\textcolor{orange}{1}$};
    \node at (2.25, \y+0.25) {$\textcolor{orange}{0}$};
    \node at (2.75, \y+0.25) {$\green{x_3}$};
    
    \draw[-stealth] (0,0) -- node[anchor=east,scale=0.75] {$\textcolor{blue}{\xi_0(x_0)}$} (-0.75,\y+0.5);
    \draw[-stealth] (0.75,0) -- node[xshift=-0.4cm,scale=0.75] {$\textcolor{red}{\xi_1(x_1)}$} (0.2,\y+0.5);
    \draw[-stealth] (1.3,0) -- node[xshift=-0.4cm,scale=0.75] {$\textcolor{orange}{\xi_2(x_2)}$} (1.7,\y+0.5);
    \draw[-stealth] (2,0) -- node[anchor=west,scale=0.75] {$\green{\xi_3(x_3)}$} (2.7,\y+0.5);
    
     \draw[-stealth] (1,\y) -- node[anchor=west] {$\mathcal{I}$} (1,\y-1);
     \def\y{-3}
    \draw[step=0.5] (-1, \y) grid (3, \y+0.5);
    \node at (-0.75, \y+0.25) {$\textcolor{blue}{0}$};
    \node at (-0.25, \y+0.25) {$\textcolor{red}{1}$};
    \node at (0.25, \y+0.25) {$\textcolor{red}{1}$};
    \node at (0.75, \y+0.25) {$\textcolor{orange}{2}$};
    \node at (1.25, \y+0.25) {$\textcolor{orange}{2}$};
    \node at (1.75, \y+0.25) {$\textcolor{orange}{2}$};
    \node at (2.25, \y+0.25) {$\textcolor{orange}{2}$};
    \node at (2.75, \y+0.25) {$\green{3}$};
    
\end{tikzpicture}}
    \caption{Example of a feature embedding $\bm{\xi}(\bm{x})\in \R^8$. Here, $x_0 $ and $x_3$
    are kept intact while $x_1$ and $x_2$ are one-hot-encoded. The bottom of the figure presents the function $\I$ that maps the index of an embedded coordinate to the index of its associated $\bm{x}$ component.}    
    \label{fig:ohe}
\end{figure}
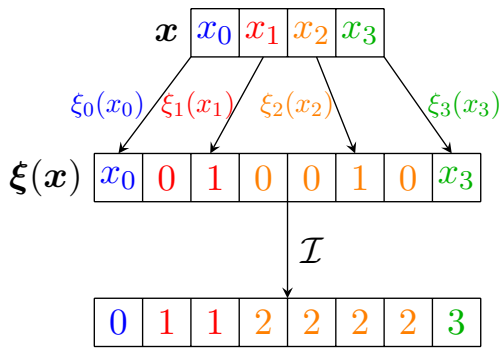

Letting $d'=\sum_{i=0}^d d_i$, we shall also define the global embedding function $\bm{\xi}:\R^d \to \R^{d'}$ that maps any given feature vector $\bm{x}$ to the concatenation
$\bm{\xi}(\bm{x}) = [\xi_0(x_0),\ldots, \xi_d(x_d)]^T$. Figure \ref{fig:ohe} shows an
example of embedding that relies on a mix of identity embeddings and one-hot encodings. For both training and prediction purposes, the vector $\bm{\xi}(\bm{x})$ stands as the input for the model in place of the feature vector $\bm{x}$. Therefore, applying TreeSHAP on the resulting Decision Trees will yield a separate value for each of the $d'$ components in the embedded vector. However the $d'$ features associated with the embedding lack interpretability and we argue it is more desirable to obtain a Shapley value for each $d$ component of the original vector $\bm{x}$.
Notice that our embedding $\bm{\xi}$ induces a surjective function $\I:[d']\to[d]$ given by
\begin{equation} 
\I(i) = \min \bigg\{j\in [d] : i < \sum_{k=0}^jd_k\bigg\}
\label{eq:I_def}
\end{equation}
encoding a partition $\{P_i\mid i\in[d]\}$ of $[d']$ where each set $P_i$ contains the indices of coordinates corresponding to feature $i$ via the embedding $\bm{\xi}$.
We therefore propose to use Partition-TreeSHAP with $\bm{\xi}(\bm{x}),\bm{\xi}(\bm{z})$ in place of $\bm{x},\bm{z}$ and the mapping $\I$ of Equation \ref{eq:I_def}. Doing so will provide a Shapley value for each of the $d$ component of the original vector.
Since feature embeddings are not supported in the current implementation of TreeSHAP from the SHAP library, we see Partition-TreeSHAP as a key contribution from a theoretical and practical perspective.

\section{Conclusion}\label{sec:conclu}

The Interventional TreeSHAP algorithm has previously revolutionized the computation of Shapley values for explaining decision tree ensembles. Indeed, before its invention, the exponential burden of
Shapley values severely limited their application to Machine Learning problems, which typically involve large numbers of features. In this work, we take a 
step back and give a presentation of the mathematics behind the success of Interventional TreeSHAP.
In doing so, our goal is both educational and research-oriented.
On the one hand, the content of this paper can serve as a reference for a class on the mathematics behind interpretable Machine Learning. To this end, we also provide a simplified C++ implementation of Interventional TreeSHAP wrapped in Python as teaching material for the methods\footnote{\url{https://github.com/gablabc/Understand_TreeSHAP}}. Our implementation is not meant to rival that of the SHAP library, which makes use of additional computational optimizations at the cost of clarity.
On the other hand, we have shown that our theoretical results on Interventional TreeSHAP can be leveraged to effortlessly adapt the Shapley value computations to other game theory
indices like Shapley-Taylor indices. Moreover, we have extended Interventional TreeSHAP to tasks where categorical features are embedded in a metric space before being fed to the model. We hope that our proof can stimulate future research in the 
application of game theory for post-hoc explanations of tree ensembles and similar models such as RuleFits.

\section*{Acknowledgements}

The authors wish to thank the DEEL project CRDPJ 537462-18 funded 
by the National Science and Engineering Research Council of Canada 
(NSERC) and the Consortium for Research and Innovation in Aerospace 
in Québec (CRIAQ), together with its industrial partners Thales Canada 
inc, Bell Textron Canada Limited, CAE inc and Bombardier inc.

\newpage

\vskip 0.2in
\bibliography{biblio}

\end{document}